\documentclass[lettersize,journal]{IEEEtran}

\usepackage[T1]{fontenc}
\usepackage{graphicx}
\usepackage{amsmath}
\usepackage{amssymb}
\usepackage{amsthm}
\usepackage{booktabs}
\usepackage{url}
\usepackage{xspace}
\usepackage[subtle]{savetrees}
\graphicspath{{./figures/}{../images/}}

\newtheorem{definition}{Definition}[section]
\newtheorem{assumption}{Assumption}[section]
\newtheorem{theorem}{Theorem}[section]

\providecommand{\texorpdfstring}[2]{#1}
\newcommand{\william}[1]{}
\newcommand{\ravi}[1]{}
\newcommand{\ravibot}[1]{}
\newcommand{\raviim}[1]{}
\newcommand{\rev}[1]{#1}
\DeclareRobustCommand{\orangechange}[1]{#1}
\newcommand{\MDP}{\textsc{MDP}\xspace}

\newcommand{\IPOMDP}{\textsc{IPOMDP}\xspace}
\newcommand{\IPOMDPs}{\textsc{IPOMDP}s\xspace}
\newcommand{\POMDP}{\textsc{POMDP}\xspace}

\newcommand{\PCIS}{\textsc{PCIS}\xspace}
\newcommand{\supplementname}{appendix\xspace}

\begin{document}

\title{Interval POMDP Shielding for \\ Imperfect-Perception Agents}

\author{William~Scarbro and Ravi~Mangal%
\thanks{This is the extended version, including appendices, of a paper presented at the International Conference on Embedded Software (EMSOFT) 2026.}%
\thanks{The authors are with the Department of Computer Science, Colorado State University, Fort Collins, Colorado, USA.}%
}

\markboth{IEEE Transactions on Computer-Aided Design of Integrated Circuits and Systems,~Vol.~XX, No.~XX, 2026 (Extended Version)}%
{Scarbro and Mangal: Interval POMDP Shielding for Imperfect-Perception Agents}

\maketitle
\IEEEpubid{0278-0070~\copyright~2026 IEEE}

\begin{abstract}
  Autonomous systems that rely on learned perception can make unsafe decisions when sensor readings are misclassified. We study shielding for this setting: given a proposed action, a shield blocks actions that could violate safety. We consider the common case where system dynamics are known but perception uncertainty must be estimated from finite labeled data. From these data we build confidence intervals for the probabilities of perception outcomes and use them to model the system as a finite Interval Partially Observable Markov Decision Process with discrete states and actions. We then propose an algorithm to compute a conservative set of beliefs over the underlying state that is consistent with the observations seen so far. 
  This enables us to construct a runtime shield that comes with a finite-horizon guarantee: with high probability over the training data, if the true perception uncertainty rates lie within the learned intervals, then every action admitted by the shield satisfies a stated lower bound on safety. Experiments on four case studies show that our shielding approach (and variants derived from it) improves the safety of the system over state-of-the-art baselines.

\end{abstract}

\begin{IEEEkeywords}
Interval Markov models, partially observable Markov decision processes (POMDPs), perception uncertainty, runtime shielding.
\end{IEEEkeywords}

\section{Introduction}
The computational capabilities unlocked by neural networks have made it feasible to build autonomous agents that use learned components to interact with their environments while pursuing complex goals. As one example, consider an autonomous aircraft taxiing system that senses the environment through a camera and must follow the runway centerline in TaxiNet-style settings \cite{xtaxinet,kadron2022case,fremont2020formal}. The agent \emph{senses} the environment through sensors, \emph{perceives} the underlying state of the system from those sensor readings via neural components, \emph{chooses} an action based on the perceived state using symbolic or neural decision logic, and \emph{executes} that action to update the system state. When the perception component misclassifies the underlying state from the sensor readings, the agent acts on incorrect information. We refer to this phenomenon as \emph{imperfect-perception}.

Autonomous agents with imperfect-perception are often deployed in safety-critical settings, so one would like all executed actions to come with formal safety guarantees. Shielding is a promising approach: a shield restricts the actions available to the agent so that safety constraints are preserved at runtime \cite{alshiekh2018safe,jansen2020safe,konighofer2020shielding}. Shields are also compatible with reinforcement-learning(RL)-based control, where they can improve the runtime safety of a deployed RL policy while preserving as much task performance as possible \cite{alshiekh2018safe,jansen2020safe,konighofer2020shielding}.

Constructing a shield with a meaningful safety guarantee requires an accurate model of perception uncertainty. Models of perception uncertainty take many forms, but at their heart they capture the probabilistic relationship between observations received by the controller and the true state of the system. When these observations are produced using a Deep Neural Network (DNN) component the perception uncertainty model acts as a probabilistic black box around the environment producing inputs for the DNN perceptor and the DNN perceptor itself. When constructing this probabilistic model, classical uncertainty models in control theory are often a poor fit because neural perception uncertainty need not follow simple parametric forms such as Gaussian noise \cite{kalman1960new}. The support-based shield of Carr \emph{et al.}, which discards emission probabilities and only retains perception relationships, is sound \cite{carr2023pomdpshielding}, but we show it may be too conservative to yield usable shields in problems of interest.

\IEEEpubidadjcol One central challenge to constructing a tenable probabilistic uncertainty model is the fact that the perception environment experienced by an agent will never exactly match the data from which the uncertainty model was built. This change manifests as a distributional shift in model probabilities. A model's design determines how robust it is to different kinds of distributional shift. One type of shift results from a change in the control policy from when the data was collected to when the agent is deployed. This changes the distribution of states the agent visits \cite{ross2011reduction}, but it does not change the probability associated with producing an observation from a particular state, called an \textit{emission probability}. There are several uncertainty modeling techniques which are not robust to \textit{control-based shift}. One is conformal wrappers, which modify the neural perceptor to be set estimators with a guaranteed lower bound on the probability of the true outcome being contained within the set \cite{angelopoulos2023conformal,vovk2005algorithmic} and have been used for shielding \cite{ScarbroImrie25}. This guarantee relies on exchangeability, a property that can be violated by control-based shift without additional correction \cite{tibshirani2019conformal,gibbs2021adaptive}. To guard against control-based shift one can isolate the individual per-state emission probabilities because these probabilities are independent of the true state distribution experienced by the agent. However, this method is not robust to a more pernicious type of distributional shift, \textit{deployment shift}, which changes the per-state emission probabilities themselves. Deployment shift is inherent to our problem setting because the perception uncertainty model is learned from finite data, guaranteeing the sampled probabilities will not exactly match their deployment realizations \cite{ovadia2019can}. Our goal is to build a model robust to control-based shift that guards against deployment shift to the extent that is feasible.

We begin with a Partially Observable Markov Decision Process (\POMDP) model of the agent (or system) with a discrete state and action spaces, which is an appropriate base model in this setting as it describes a control system with limited access to state data \cite{kaelbling1998planning}. In this paper we work throughout with finite discrete models: discrete latent states, discrete actions, and a finite set of discrete observations output by the perception stack. Specifically, a POMDP composes a \textit{dynamics model} of how the state space evolves, with a \textit{perception uncertainty model} of per-state emission probabilities \cite{pasareanu2023closedloop,schaefers2026perceptionbased}. We address a common regime where the system dynamics are modeled exactly, but the perception uncertainty model is built from finite labeled data. Our method constructs confidence bounds on the emission probabilities and uses them to define an Interval Partially Observable Markov decision process (\IPOMDP), whose perception uncertainty model uses intervals for these emission probabilities rather than fixed values.

This follows the broader robust/imprecise probabilistic modeling perspective of bounded-parameter MDPs and imprecise hidden Markov models \cite{givan2000bounded,nilim2005robust,zaffalon2009imprecise}. These intervals make the model robust to deployment shift because they anticipate the realized emission probabilities varying from scalar expectations. Like POMDPs, the perception uncertainty model in the IPOMDP is independent of the control policy used to collect the data for estimating emission probabilities/intervals, but unlike the base POMDP model, it also carries a Probably Approximately Correct (PAC)-style model-correctness guarantee: with probability at least $\lambda$, the true perception uncertainty model lies in the admissible set of perception uncertainty models induced by the interval bounds \cite{clopper1934use}.

Given this interval perception model, we develop a shielding method that starts from a perfect-perception shield, where the agent is assumed to have access to the true state of the system, and lifts it to the imperfect-perception setting. The underlying perfect-perception shield may come from Probabilistically Controlled Invariant Sets (\PCIS) \cite{gao2021pcis}, barrier-certificate methods \cite{ames2019control}, or other state-space shield-synthesis procedures. Starting from such a shield is deliberate: computing a shield directly in belief space under interval observations scales poorly because the belief update is set-valued and normalized, producing high-dimensional nonconvex reachable sets \cite{noack2008convexsets}. Our method therefore propagates over-approximations of reachable beliefs and tests action admissibility to the shield through worst-case linear objectives over that envelope, solved with linear programming. The resulting design is modular by allowing various perfect-perception shield designs, conservative by construction, and separates dynamics-level safety shielding from reasoning about perception uncertainty.

We evaluate the method on four benchmark domains: TaxiNet, Obstacle, CartPole, and Refuel. The experiments address two main empirical questions. First, across these benchmarks, how do different shields compare in safety and usability under imperfect-perception? To answer this, we compare our proposed shield with two under-approximate variants of our shield construction and two other baselines including the support-based shield of Carr \emph{et al.} \cite{carr2023pomdpshielding}. Second, how conservative is the envelope abstraction used by the proposed method? To answer this, we compare the envelope's over-approximation of reachable beliefs against forward-sampled under-approximations. Together, these experiments show that interval belief-envelope shielding is most useful in partially observable regimes where shields constructed assuming point emission probabilities are too optimistic and the support-based shield becomes too coarse, while also revealing benchmark-dependent computational limits. \rev{These results show a scalability tradeoff: the sound over-approximate shield construction is tractable on smaller \IPOMDP{}s but does not scale to the largest models, while a sampling-based under-approximate variant inspired by the same construction scales further but is not given the same soundness guarantee here.}

\subsection*{Contributions}
We make the following contributions:
\begin{enumerate}
\item An \IPOMDP formulation of the imperfect-perception system under study, including finite-data observation uncertainty and the associated shielding problem for finite discrete models.
\item A runtime shield construction for \IPOMDP{}s that lifts a perfect-perception shield to the imperfect-perception setting.
\item \rev{A sound but expensive linear-programming/linear-fractional-programming-based method for over-approximating reachable beliefs in an \IPOMDP, as well as a cheap, under-approximate, sampling-based variant, and a method for using those belief envelopes for conservative action admissibility.}
\item Empirical evidence across four benchmark domains (TaxiNet, Obstacle, CartPole, and Refuel), including a comparison among Observation, Single-Belief, Envelope, Fwd-Sampling, and the support-based shield, and a quantitative study of envelope conservativeness.
\item A probabilistic-automata/categorical interpretation of the construction that explains the convex abstraction used for interval belief reasoning; available in the \supplementname.
\end{enumerate}

\section{Preliminaries}
\label{sec:prelims}
\subsection{System Model}
We begin with a discrete Markov Decision Process (MDP) model of the agent/system,
\[
\mathcal{M}=(S,A,T),
\]
where $S$ is a finite state set, $A$ is a finite action set, and $T:S\times A\times S\to[0,1]$ is the transition kernel. At time $t$, the system is in  state $s_t\in S$, the controller chooses action $a_t\in A$, and the next state is sampled from $T(s_t,a_t,\cdot)$.
\rev{For a fixed action $a$, we write $T_a$ for the $|S|\times |S|$ transition matrix with entries $(T_a)_{s,s'}=T(s,a,s')$.}
\paragraph*{Perfect perception}
Under perfect perception, the controller observes $s_t$ exactly before selecting $a_t$. In that setting the decision problem is an MDP over the true state space $S$.

\paragraph*{Imperfect perception and raw sensor inputs}
In learned autonomy, the controller typically does not observe $s_t$ directly. Instead, a perception component receives a raw sensor input $y_t\in\mathcal{Y}$, where $\mathcal{Y}$ may be high-dimensional, for example an image space. Although $y_t$ is the direct result of $s_t$, the relationship which produces $y_t$ involves environmental conditions and therefore is extremely difficult to model directly. A learned perception component, such as a deep neural network, then converts $y_t$ into a finite symbolic output. This output is what the controller \textit{observes} in the system model, and forms the basis of the belief of state space occupation.

\subsection{System Model with Probabilistic Abstraction}
Let $O$ be a finite set of discrete observations. Rather than model the raw sensor domain $\mathcal{Y}$ directly, we apply previous work \cite{pasareanu2023closedloop, schaefers2026perceptionbased} to summarize the perception pipeline by a conditional distribution from the true state ($s_t$) of the system at time $t$ to discrete perception outputs,
\[
Z(o\mid s)=\Pr[\text{perception outputs }o \mid s_t=s], \qquad o\in O.
\]
The type $Z: S \times O \rightarrow [0,1]$ fits a POMDP-style model of the imperfect-perception system: the true state evolves according to $T$, remains hidden from the controller, and the controller instead receives observations drawn from $Z(\cdot \mid s)$. The abstraction avoids carrying the full high-dimensional sensor domain through the system model: all sensor complexity is compressed into the finite perception uncertainty model $Z$, which records only the probabilities of the perception outputs relevant to decision making. Throughout the paper we stay in this finite discrete setting, so the shield reasons over discrete latent states, discrete actions, and discrete observation symbols rather than continuous sensor values directly.
This probabilistic abstraction is close in spirit to recent approaches that also replace the raw sensor domain by a compact stochastic interface between true state and downstream reasoning. Sch\"afers \emph{et al.} use a learned perception uncertainty model to produce probabilistic state summaries for POMDP belief updates \cite{schaefers2026perceptionbased}, while P\u{a}s\u{a}reanu \emph{et al.} use confusion-matrix-derived probabilistic abstractions to enable closed-loop probabilistic analysis of vision-based autonomy \cite{pasareanu2023closedloop}. Our work shares this abstraction viewpoint, but then develops an interval-valued version of $Z$ to support conservative shielding under finite-data uncertainty.

\paragraph*{POMDP}
With a fixed (perception uncertainty model) $Z$, the imperfect-perception system is a POMDP
\[
\mathcal{P}=(S,A,O,T,Z).
\]
We write $\Delta(S)$ for the probability simplex over $S$. For a history $h_t$ of action, observation pairs, a belief $b_t \in \Delta(S)$ is the posterior distribution over the true state given that history, i.e., $$b_t(s) = \Pr [s_t = s | h_t].$$ Given prior belief $b_t\in\Delta(S)$, action $a_t$, and received observation $o_{t+1}$, the controller updates its belief by Bayesian filtering \cite{kaelbling1998planning}. The state still evolves according to $T$; the uncertainty comes from not observing $s_t$ directly and instead receiving only $o_t$ through the perception uncertainty model $Z$.

\subsection{Interval POMDP System Model}
In practice we must estimate the true $Z$ from finite data. While one can form point estimates of the emission probabilities, finite-sample error can make those estimates arbitrarily misleading, especially for rarely observed state-observation pairs. We therefore replace the single perception uncertainty model by interval bounds $Z^-,Z^+:S\times O\to[0,1]$ \cite{pasareanu2023closedloop,cleaveland2025conservative} and define the admissible set of perception uncertainty models
\[
\mathcal{Z}=\left\{
Z~\middle|~
\begin{aligned}
&\forall s,o:\; Z^-(o\mid s)\le Z(o\mid s)\le Z^+(o\mid s),\\
&\forall s:\; \sum_o Z(o\mid s)=1.
\end{aligned}
\right\}.
\]

\begin{definition}[\IPOMDP]
An \IPOMDP is a tuple
\[
\mathcal{I}=(S,A,O,T,Z^-,Z^+),
\]
where $(S,A,T)$ is the initial MDP system and $\mathcal{Z}$ is the admissible set of perception uncertainty models defined above.
\end{definition}

Given history $h_t=(a_0,o_1,\dots,a_{t-1},o_t)$ and prior belief $b_0\in\Delta(S)$, each $Z\in\mathcal{Z}$ induces a posterior belief $b_t^Z$. Over all elements of $\mathcal{Z}$, this forms a set of beliefs:
\[
\mathcal{B}_t(b_0,h_t)=\{b_t^Z: Z\in\mathcal{Z}\}.
\]
This definition fixes a single unknown $Z$ across the whole history $h_t$. For shielding we use a more conservative semantics that allows the perception uncertainty model to vary within $\mathcal{Z}$ at each step; as a result our belief-envelope construction is a superset of the belief set produced by the fixed semantics. We adopt this stronger per-step-varying semantics because the safety guarantee should remain valid for any admissible realization of the perception error at each step, rather than depending on one fixed but unknown kernel being sampled once and then held constant for the whole trajectory. There is also a computational reason for this: the deployed method propagates a compact belief envelope at runtime, which scales much better than recomputing the exact shield from the entire action-observation history at each step.

In an ordinary POMDP, a single posterior belief is updated by Bayesian filtering. In an \IPOMDP, however, interval uncertainty in the perception uncertainty model means that a whole set of posterior beliefs can be consistent with the same history, so we will need a separate algorithm to propagate a conservative belief envelope over a history. 
\paragraph*{PAC guarantee}
When the interval bounds are constructed from finite labeled data using confidence intervals, the resulting \IPOMDP carries a dataset-level model-correctness guarantee. In particular, if the true perception uncertainty model is denoted by $Z^\ast$, then our construction ensures
\[
\Pr[Z^\ast\in\mathcal{Z}] \ge \lambda,
\]
for a user-chosen confidence level $\lambda$. Section~\ref{sec:interval-construction-and-pac} gives the concrete Clopper--Pearson construction and the corresponding union-bound argument.

\subsection{Shielding}
Our objective is to lift a state-based (perfect-perception) safety shield to the imperfect-perception setting rather than synthesize an entirely new shield.

\begin{definition}[Perfect-Perception Shield]
A perfect-perception shield is a map $\Omega:S\to 2^A$ assigning each true state the set of admissible actions.
\end{definition}

The shield is intended to enforce a safety specification defined by a designated safe set of states specified logically. It removes actions whose probability of leaving that safe set exceeds the chosen guarantee notion. In other words, we consider safety specifications that can be expressed using the safety fragment of Probabilistic Computational Tree Logic (PCTL).

This definition provides a minimal interface -- a map from true states to a set of actions -- and is agnostic to how the shield is synthesized. It may come from reachability analysis, barrier-certificate methods, controlled invariant sets, or any other state-space shield construction.

\begin{definition}[Imperfect-Perception Shield]
An imperfect-perception shield is a map $\Xi:\bigcup_{t\ge 0}\mathcal{H}_t\to 2^A$ from finite action-observation histories to sets of actions.
\end{definition}

Here $\mathcal{H}_t$ denotes the set of length-$t$ action-observation histories, e.g. $h_t = (a_0,o_1,\dots,a_{t-1},o_t) \in (A \times O)^t$. This is the trajectory-level shielding object: given the history available at time $t$, it returns the actions permitted at that time.

A runtime shield, by contrast, is the runtime mechanism used to realize such a map. It maintains an internal summary of state uncertainty, updates that summary after each action-observation pair, and computes the next permitted-action set. In this way, the runtime shield induces an imperfect-perception shield over histories. This distinction matters because although the shield itself is fundamentally history-based, our implementation is designed to avoid storing and recomputing over the full history explicitly.

To define the runtime shield, first define the action-safety indicator from the perfect-perception shield
\[
\chi_\Omega(s,a)=
\begin{cases}
1, & a\in\Omega(s),\\
0, & a\notin\Omega(s).
\end{cases}
\]

\rev{\noindent For each action $a$, let
\[
\phi_a(b)=\sum_{s\in S}b(s)\chi_\Omega(s,a),
\]
so $\phi_a(b)$ is the probability, under belief $b$, that action $a$ is allowed by the perfect-perception shield.}

For a history $h_t$, initial belief $b_0$, and threshold $\beta\in[0,1]$, the exact lifted shield induced by $\Omega$ is
\begin{equation}
\label{eq:runtime-shield}
\rev{\Xi^\star(h_t)=\left\{a\in A:\inf_{b\in\mathcal{B}_t(b_0,h_t)}\phi_a(b)\ge\beta\right\}.}
\end{equation}
That is, an action is permitted only if it remains $\Omega$-admissible with probability at least $\beta$ under every belief in the interval-POMDP reachable-belief set $\mathcal{B}_t(b_0,h_t)$. Our method approximates this exact object conservatively using belief envelopes that over-approximate that reachable-belief set.

\paragraph*{PCIS construction}
We use Probabilistically Controlled Invariant Sets (\PCIS) as the main source of perfect-perception shields because they provide a composable stochastic invariance certificate \cite{gao2021pcis}. In that framework, an $N$-step $\epsilon$-PCIS is a set of states $C$ such that each state in $C$ admits a policy that keeps the trajectory inside $C$ for $N$ steps with probability at least $\epsilon$; the infinite-horizon variant requires this for all future times.

To obtain a perfect-perception shield from a certified invariant set $C\subseteq S$, we define
\begin{equation}
\label{eq:omega-pcis}
\Omega(s)=\left\{a\in A:\sum_{s'\in C} T(s,a,s')\ge \gamma\right\},
\end{equation}
with threshold $\gamma\in(0,1]$. Thus, PCIS supplies the invariant core $C$, and $\Omega$ is the induced state-action shield under the known dynamics kernel $T$. The remainder of the paper only requires the interface $\Omega:S\to 2^A$, so the runtime lifting method is compatible with other perfect-perception shield constructions as well.

\section{Method}
\label{sec:method}
\orangechange{Because $\Omega$ is synthesized offline in true-state space, the runtime shield maintains a conservative belief set consistent with the action-observation history and admits only actions safe for every belief in that set. We make this tractable by reducing belief-envelope propagation to a linear optimization problem over a template domain, yielding an LP-solvable over-approximation.}

Our goal is to propagate not just one belief but an entire belief envelope through an action-observation update, so we begin with the algorithm to propagate a single belief, the ordinary POMDP forward recursion, equivalently the Bayes-filter update \cite{rabiner1989tutorial}. Given a prior belief $b$ and action-observation pair $(a,o)$, the unnormalized posterior mass is
\[
u_{s'} = Z(o\mid s')\sum_{s\in S} T(s,a,s')\,b(s),
\]
and the posterior belief is obtained by normalizing $u$. Equivalently, in vector form,
\[
y=T_a^\top b,\qquad u_s = y_s Z(o\mid s),\qquad b'_s=\frac{u_s}{\sum_j u_j}.
\]
\orangechange{Our envelope update follows exactly this recursion, but simultaneously for every prior belief in $\widehat{\mathcal{B}}_t$ and every perception uncertainty model in $\mathcal{Z}$. The LP conservatively applies this recursion over these sets to construct an outer approximation of the reachable beliefs.}

\subsection{Exact Admissibility Objective}
\rev{Using the belief-level safety score $\phi_a$ from Section~\ref{sec:prelims},} given history $h_t$, exact admissibility at threshold $\beta$ is
\[
a\in\Xi^\star(h_t)\iff\inf_{b\in\mathcal{B}_t(b_0,h_t)}\phi_a(b)\ge\beta.
\]
This is exactly the shield defined in Equation \eqref{eq:runtime-shield}, written using $\phi_a$ for compactness. Since $\mathcal{B}_t$ is generally intractable to represent exactly, we replace it with the conservative envelope $\widehat{\mathcal{B}}_t$ in the deployed runtime shield below.

\subsection{Why Exact Convex-Hull Propagation is Intractable}
The exact lifted shield depends on the reachable-belief set $\mathcal{B}_t$, so the main algorithmic question is how to represent and propagate that set online. A natural first idea is to propagate the exact reachable set itself, but that approach is intractable.

The 
dynamical state update
step 
of the IPOMDP that applies the transition kernel $T$
is affine in the belief, so it preserves convexity by itself. The difficulty appears in the observation update: the posterior is obtained by multiplying by emission probability intervals and then normalizing by the total evidence mass, which introduces bilinear and linear-fractional terms.
As a result, even when the current belief set is convex, the exact set of reachable posteriors can become nonconvex after one update. Intuitively, an extreme point of this reachable set corresponds to a worst-case combination of two ingredients: a corner of the current belief set and a particular choice of endpoint 
probabilities
within the observation intervals. After one step, each such corner case can branch into many new corner cases, because a different endpoint choice can be worst-case for each state and observation. Repeating this process over time causes the number of relevant extreme points to grow exponentially with the horizon. This vertex explosion makes exact convex-hull propagation scale very poorly in both memory and runtime, which is why it is not suitable for online shielding \cite{noack2008convexsets}.

\subsection{Outer Polytope Propagation}
\label{sec:main-lp-formulation}
We describe how to propagate the belief envelope through the IPOMDP for a single step, i.e., a single application of the dynamics and observation update under a fixed action $a$.
We represent each belief envelope by a \emph{template polytope}: a polytope described by a fixed set of linear directions chosen in advance, together with the worst-case bound attained in each direction. Equivalently, the template matrix specifies which halfspaces we track, and the vector of offsets records how far the reachable set extends along those directions. At time $t$, maintain
\[
\widehat{\mathcal{B}}_t=\{b\in\mathbb{R}^{|S|}:A_tb\le d_t,\;b\ge 0,\;\mathbf{1}^\top b=1\},\quad \mathcal{B}_t\subseteq\widehat{\mathcal{B}}_t.
\]
\rev{Let \(n=|S|\). If the current envelope has \(m_t\) template rows, then \(A_t\in\mathbb{R}^{m_t\times n}\) and \(d_t\in\mathbb{R}^{m_t}\).}

A key structural reason for using convex envelopes is that the 
dynamical state update step is affine in the current belief, so sets of reachable pre-normalization beliefs remain convex; the loss of convexity enters only 
\rev{after the observation update, where the uncertain products $y_s w_s$ are coupled with the Bayes normalization by the total evidence mass.}
Our goal is not to propagate one posterior belief, but an entire prior belief envelope through an action-observation update. The difficulty is that the exact update does not fit directly into LP: uncertain observation 
probabilities
create bilinear products, and Bayes normalization introduces division by the total evidence mass. The purpose of the next steps is therefore to rewrite the propagation problem into a conservative LP-solvable form.

\rev{Steps 1--3 describe the update for one candidate prior belief $b\in\widehat{\mathcal{B}}_t$ and give the corresponding LP construction immediately after the intuition. Step 4 then assembles the final belief envelope by optimizing template directions over all prior beliefs and all admissible observation-probability choices. We write each LP in the standard form
\[
\max_x\;\; r^\top x\quad\mathrm{s.t.}\quad Gx\le h,\;Ex=f,\;\underline{x}\le x\le\overline{x}.
\]
Here \(Gx\le h\) collects inequality constraints, \(Ex=f\) collects equality constraints, \(\underline{x}\le x\le\overline{x}\) gives simple bounds, and \(r\) is the objective vector; the complete LP is the tuple \((r,E,f,G,h,\underline{x},\overline{x})\).}

\textbf{Step 1 (Dynamical state update):} Assuming that the agent chooses action $a$, we can propagate a prior belief through the known dynamics as,
\[
y=T_a^\top b.
\]
Here $T_a$ is the transition matrix for action $a$. This step computes the next-state belief based on the known dynamics function. The vector $y$ is the 
belief after the action $a$ but before conditioning on the observation. This map is affine in $b$, so convexity is preserved and the update remains LP-compatible.
\rev{In the LP, this step contributes the prior-normalization and state-update constraints \(\mathbf{1}^\top b=1\) and \(y=T_a^\top b\). Because \(T_a\) is row-stochastic, meaning its entries are nonnegative and each row sums to one, these constraints preserve total mass: \(b\ge0\) and \(\mathbf{1}^\top b=1\) make the prior a distribution, and \(y\) is also a distribution.}

\textbf{Step 2 (Observation update without normalization):}
At this point the current observation $o$ arrives. Assuming that $w$ denotes the observation probabilities, the updated unnormalized belief is given by,
\[
Z^-(o\mid s)\le w_s\le Z^+(o\mid s),\qquad u_s = y_s w_s\quad \forall s.
\]
 For each state $s$, the scalar $w_s$ denotes the observation %
 probability
 $Z(o\mid s)$ chosen from the admissible interval for that state-observation pair. The quantity $u_s$ is the resulting unnormalized posterior mass contributed by state $s$: 
 post-dynamical update belief for state $s$,
  $y_s$, weighted by the probability that state $s$ produces observation $o$. In an ordinary POMDP 
 $w_s$ would be fixed; in an \IPOMDP\ $w_s$ may vary within the interval, so we must account for all admissible choices.
The scalar $\sum_j u_j$ is the total evidence mass used in the normalization step below. So $u_s$ is intentionally state-wise, while $\sum_j u_j$ is the quantity obtained after summing those state-wise contributions across all $s$.

The coupling term is bilinear, so we replace $u_s=y_sw_s$ by its McCormick envelope \cite{mccormick1976} over known bounds on $y_s$ and $w_s$. This gives a sound linear 
over-approximation
of feasible $(u,y,w)$ tuples. Intuitively, the McCormick envelope is the standard LP relaxation of a bounded product: it replaces the exact bilinear equality by linear inequalities that still contain every feasible product value. In this sense, the resulting optimization is an LP relaxation of the forward recursion over the whole belief envelope rather than a different update rule.
\rev{The construction introduces the observation-probability vector \(w\) and unnormalized mass vector \(u\), and collects the pre-normalization variables into the LP decision vector in a fixed \emph{block order}, meaning the variable groups are concatenated in the same order used by the LP constraint matrices:
\[
z=(b^\top,y^\top,w^\top,u^\top)^\top\in\mathbb{R}^{4n}.
\]
With this block order, the equality constraints from Step 1 are
\[
E_0=\begin{bmatrix}\mathbf{1}^\top&0&0&0\\ -T_a^\top&I&0&0\end{bmatrix},
\qquad
f_0=\begin{bmatrix}1\\\mathbf{0}\end{bmatrix},
\]
so \(E_0z=f_0\) enforces \(\mathbf{1}^\top b=1\) and \(y=T_a^\top b\). The prior envelope itself enters as the first block row of the inequality matrix below. The construction then restricts the interval observation row to the observed output \(o\) while respecting row-stochasticity:
\begin{align*}
w_s^-&=\max\;\{Z^-(o\mid s),1-\sum_{o'\ne o}Z^+(o'\mid s)\},\\
w_s^+&=\min\;\{Z^+(o\mid s),1-\sum_{o'\ne o}Z^-(o'\mid s)\}.
\end{align*}
For each state, compute state-update bounds
\[
L_s=\min_{b\in\widehat{\mathcal{B}}_t}\; e_s^\top T_a^\top b,\qquad
U_s=\max_{b\in\widehat{\mathcal{B}}_t}\; e_s^\top T_a^\top b,
\]
where \(e_s\) is the standard basis vector. These bounds define the McCormick relaxation \cite{mccormick1976} of \(u_s=y_sw_s\) over \(y_s\in[L_s,U_s]\) and \(w_s\in[w_s^-,w_s^+]\). With block columns ordered as \((b,y,w,u)\), writing \(D_v\) for a diagonal matrix and \(\odot\) for componentwise multiplication, the inequality data for these unnormalized constraints are
\[
G_0=\begin{bmatrix}
A_t&0&0&0\\
0&D_{w^-}&D_L&-I\\
0&D_{w^+}&D_U&-I\\
0&-D_{w^-}&-D_U&I\\
0&-D_{w^+}&-D_L&I
\end{bmatrix},\quad
h_0=\begin{bmatrix}
d_t\\ L\odot w^-\\ U\odot w^+\\ -U\odot w^-\\ -L\odot w^+
\end{bmatrix}.
\]
The first block row is the prior-envelope constraint \(A_tb\le d_t\); the last four block rows are the McCormick inequalities. Together with \(E_0z=f_0\), the simple bounds are
\[
\underline{z}=(\mathbf{0}^\top,L^\top,(w^-)^\top,\mathbf{0}^\top)^\top,\quad
\overline{z}=(\mathbf{1}^\top,U^\top,(w^+)^\top,\mathbf{1}^\top)^\top.
\]
Thus \(G_0z\le h_0\), \(E_0z=f_0\), and the simple bounds define the unnormalized feasible set.}

\textbf{Step 3 (Normalization):} Posterior belief is fractional,
\[
b'_s=\frac{u_s}{\sum_j u_j}.
\]
This is the usual Bayes normalization: divide each unnormalized posterior mass by the total evidence mass so the posterior sums to one. That division is exactly what breaks linearity. To recover LP solvability, we apply Charnes--Cooper \cite{charnes1962programming}: introduce a scaling variable for $\left(\sum_j u_j\right)^{-1}$ and scale decision variables so normalization constraints become linear.

A useful way to view this step is to work first with unnormalized posterior masses, where the update is affine, and then recover normalized beliefs through a linear-fractional normalization handled by Charnes--Cooper \cite{charnes1962programming}. Intuitively, this rewrites the posterior update so that division by the total evidence mass is absorbed into the new variables, allowing normalized posterior queries to be solved by LP without changing their optimal values.
\rev{For a template direction \(c\), the normalized objective before transformation is the linear-fractional program
\[
\max\;\frac{c^\top u}{\mathbf{1}^\top u}
\quad\mathrm{s.t.}\quad
G_0z\le h_0,\;E_0z=f_0,\;\underline{z}\le z\le\overline{z}.
\]
Choose an arbitrary scale constant \(\kappa>0\), let \(\eta=\mathbf{1}^\top u\), and define \(\tilde z=(\kappa/\eta)z\) and \(t=\kappa/\eta\). With \(x=(\tilde z^\top,t)^\top\), \(d_u\in\mathbb{R}^{4n}\) selecting the \(u\)-block, and block order \((\tilde b,\tilde y,\tilde w,\tilde u,t)\), the standard LP tuple is
\[
r=(\mathbf{0}^\top,\mathbf{0}^\top,\mathbf{0}^\top,c^\top,0)^\top,
\]
\[
G=\begin{bmatrix}G_0&-h_0\\ I&-\overline z\\ -I&\underline z\end{bmatrix},\quad
h=\mathbf{0},\quad
E=\begin{bmatrix}E_0&-f_0\\ d_u^\top&0\end{bmatrix},\quad
f=\begin{bmatrix}\mathbf{0}\\\kappa\end{bmatrix},
\]
with \(\underline{x}=(-\infty,\ldots,-\infty,0)^\top\) and \(\overline{x}=(\infty,\ldots,\infty,\infty)^\top\). The inequality matrix \(G\) is the scaled copy of \(G_0z\le h_0\) together with the scaled variable bounds \(\underline z\le z\le\overline z\). The equality \(d_u^\top\tilde z=\kappa\) fixes the normalization scale. The LP objective equals \(\kappa c^\top u/(\mathbf{1}^\top u)\), so dividing the optimum by \(\kappa\) gives the actual template offset.}

\textbf{Step 4 (Template projection):} \orangechange{Lift the single-belief update to the whole envelope. Fix the next template matrix \(A_{t+1}\in\mathbb{R}^{m_{t+1}\times n}\), with one row per next-envelope template direction, and define}
\[
\widehat{\mathcal{B}}_{t+1}=\{b'\in\mathbb{R}^{|S|}:A_{t+1}b'\le d_{t+1},\;b'\ge 0,\;\mathbf{1}^\top b'=1\},
\]
where \rev{each component of the offset vector $d_{t+1}$} is computed by one LP (with $c_i^\top$ the $i$th row of $A_{t+1}$):
\[
d_{t+1,i}=\sup c_i^\top b',
\]
\rev{In this optimization, $A_t,d_t,A_{t+1},a,o,Z^-,Z^+$ are fixed inputs. The LP variables are the prior belief $b$, predicted belief $y$, admissible observation-probability vector $w$, unnormalized mass $u$, and the auxiliary Charnes--Cooper variables that encode the normalized posterior. Lower template bounds are obtained by using direction \(-c_i\) and negating the resulting optimum.}
Step 4 is where the whole prior envelope is propagated: each LP ranges over every prior belief in $\widehat{\mathcal{B}}_t$ together with every admissible
choice for the observation probabilities from the intervals
and returns the worst-case extent of the reachable posterior set in one template direction.
\rev{After solving this LP for every row \(c_i^\top\) of \(A_{t+1}\), the offsets \(d_{t+1,i}\) assemble the next envelope.}

Facet-wise maximization gives the tightest bound along each chosen template direction while keeping runtime propagation tractable.
Geometrically, each row of the template matrix asks how far the reachable posterior set extends in one chosen direction in belief space; for example, with coordinate directions, the template polytope simply records upper and lower bounds on each state's belief mass.

\orangechange{Facet-wise maximization costs one LP per template direction. If $n=|S|$ and $m_t$ is the number of constraints in $A_t b\le d_t$, a typical lifted LP has $O(n)$ variables (roughly $4n$ to $5n$) and $O(m_t+n)$ constraints, plus four McCormick inequalities per state. With $k_{t+1}$ template rows, propagation solves $k_{t+1}$ LPs per action-observation update, so runtime scales mainly with template facets and action-observation branching; our runtime measurements confirm this empirically (see Appendix~\ref{sec:inference-timing}).}

The shape of $\widehat{\mathcal{B}}_{t+1}$ is a template polytope in $H$-representation: an intersection of halfspaces with the probability simplex. This choice makes worst-case linear safety queries and propagation share the same LP machinery. In this work we use hypercube templates, i.e., coordinate-wise bounds on belief mass. Other template families, such as zonotopes, ellipsoids, or vertex-based polytopes, could be used instead, with the usual tradeoff that richer templates can tighten the envelope at higher computational cost.

\subsection{Conservative Shield}
\noindent Define lower bound
\[
\underline{p}_t(a)=\inf_{b\in\widehat{\mathcal{B}}_t}\phi_a(b).
\]
\rev{Since \(\phi_a(b)=\sum_s b_s\chi_\Omega(s,a)\), this score is the ordinary LP
\[
\underline{p}_t(a)=
\min_b\;\sum_s b_s\chi_\Omega(s,a)
\quad\mathrm{s.t.}\quad
A_tb\le d_t,\;b\ge0,\;\mathbf{1}^\top b=1,
\]
with no Charnes--Cooper variables.}
The deployed shield is
\[
\widehat{\Xi}(h_t)=\{a\in A: \underline{p}_t(a)\ge\beta\}.
\]
The replacement of $\mathcal{B}_t$ with $\widehat{\mathcal{B}}_t$ in the infimum calculation is conservative because $\widehat{\mathcal{B}}_t$ contains every belief in $\mathcal{B}_t$, so minimizing over the larger set can only decrease the safety score and therefore can only remove actions, not admit unsound ones.

\subsection{\texorpdfstring{\rev{Two-State Worked Example}}{Two-State Worked Example}}
\rev{As a numerical check, take an IPOMDP with two states \(s=1,2\), identity dynamics, and a current belief envelope \(\widehat{\mathcal{B}}=\{b:0.4\le b_1\le0.6,\ b_2=1-b_1\}\), where \(b_i\) is the belief mass on state \(i\). Let the observation alphabet be \(\{o,\bar o\}\), and suppose the observed output at this step is \(o\). Suppose the raw interval constraints are
\[
\begin{array}{c|cc}
 & Z(o\mid s) & Z(\bar o\mid s)\\ \hline
s=1 & [0.65,0.95] & [0.10,0.30]\\
s=2 & [0.10,0.50] & [0.60,0.80]
\end{array}
\]
with row-stochastic observation rows. Projecting onto the received symbol gives \(w_1\in[0.7,0.9]\) and \(w_2\in[0.2,0.4]\). Then
\[
b'_1=\frac{b_1w_1}{b_1w_1+(1-b_1)w_2}.
\]
Using coordinate template directions \(c=e_1\) and \(c=-e_1\), this expression is increasing in \(b_1\) and \(w_1\), and decreasing in \(w_2\). Thus the lower endpoint, equivalently the solution of the LP in direction \(-e_1\), occurs at \(b_1=0.4,w_1=0.7,w_2=0.4\):
\[
(b'_1)^{\min}=\frac{0.4\cdot0.7}{0.4\cdot0.7+0.6\cdot0.4}=\frac{7}{13}\approx0.538.
\]
With \(\kappa=1\), this Charnes--Cooper solution has evidence \(\eta=0.52\), scale \(t=25/13\), and posterior \(b'=(7/13,6/13)\).
The upper endpoint occurs at \(b_1=0.6,w_1=0.9,w_2=0.2\), where the row-stochastic complementary probabilities are \(Z(\bar o\mid1)=0.10\) and \(Z(\bar o\mid2)=0.80\):
\[
(b'_1)^{\max}=\frac{0.6\cdot0.9}{0.6\cdot0.9+0.4\cdot0.2}=\frac{27}{31}\approx0.871.
\]
The corresponding Charnes--Cooper solution has evidence \(\eta=0.62\), scale \(t=50/31\), and posterior \(b'=(27/31,4/31)\). Thus the propagated coordinate envelope is \(0.538\le b'_1\le0.871\), \(b'_2=1-b'_1\). If action \(a\) is safe only in state 1, then \(\phi_a(b')=b'_1\), the conservative score is \(7/13\), and the lifted shield admits \(a\) exactly for thresholds \(\beta\le7/13\), blocking it for thresholds such as \(\beta=0.6\).}

\section{Interval Construction and PAC Correctness}
\label{sec:interval-construction-and-pac}
We now show how the interval bounds are constructed from finite labeled data and why the resulting admissible set of perception uncertainty models contains the true perception uncertainty model with high probability. This section supplies the statistical side of the method: the previous sections defined how shielding uses the admissible set of perception uncertainty models online, and here we explain how those interval bounds are learned offline with a finite-sample correctness guarantee.

\subsection{Clopper--Pearson Construction}
For each true state $s$ and observation $o$, let $n_s$ be the number of samples with true label $s$, and $k_{s,o}$ the count observed as $o$. We use these counts to estimate the emission probability $Z(o\mid s)$ and to construct a confidence interval for that probability. For confidence level $1-\alpha_{s,o}$, the corresponding Clopper--Pearson bounds are
\begin{align}
Z^-(o\mid s) &= \mathrm{BetaInv}\!\left(\frac{\alpha_{s,o}}{2}; k_{s,o}, n_s-k_{s,o}+1\right),\\
Z^+(o\mid s) &= \mathrm{BetaInv}\!\left(1-\frac{\alpha_{s,o}}{2}; k_{s,o}+1, n_s-k_{s,o}\right),
\end{align}
where $\mathrm{BetaInv}(q;a,b)$ denotes the $q$-quantile (inverse CDF) of a $\mathrm{Beta}(a,b)$ distribution. These are the exact equal-tail binomial confidence bounds for the entry $Z(o\mid s)$. For the edge cases, we use the standard Clopper--Pearson conventions: if $k_{s,o}=0$, then $Z^-(o\mid s)=0$ and $Z^+(o\mid s)=\mathrm{BetaInv}\!\left(1-\frac{\alpha_{s,o}}{2};1,n_s\right)$; if $k_{s,o}=n_s$, then $Z^-(o\mid s)=\mathrm{BetaInv}\!\left(\frac{\alpha_{s,o}}{2};n_s,1\right)$ and $Z^+(o\mid s)=1$.
These interval endpoints define the admissible set of perception uncertainty models $\mathcal{Z}$ together with the simplex constraints $\sum_o Z(o\mid s)=1$ for each state $s$. In the sequel we work directly with the interval bounds $Z^-(o\mid s)$ and $Z^+(o\mid s)$ rather than constructing $\mathcal{Z}$ explicitly.

We use Clopper--Pearson rather than Wilson intervals because Theorem~\ref{thm:outer-pcis-safety} later needs a finite-sample guarantee that, with probability at least $\lambda$ over the random training dataset, the admissible set of perception uncertainty models contains the true perception uncertainty model. To obtain that dataset-level containment statement, we first need valid entry-wise coverage for each $(s,o)$ before applying a union bound. Clopper--Pearson provides exact binomial coverage for every sample size and parameter value, whereas Wilson intervals can undercover in finite samples, especially near the boundary \cite{clopper1934use,wilson1927,brown2001interval}. We therefore prefer Clopper--Pearson because guaranteed coverage is needed for the later PAC-style containment result, even at the cost of wider intervals. The practical consequence is that the learned interval bounds, and hence the deployed shield, become more conservative: we trade some permissiveness for a stronger finite-sample correctness guarantee.

\subsection{From Intervals to Outer Probability $\lambda$}
This subsection converts the entry-wise confidence intervals above into the dataset-level correctness probability used throughout the paper. The relevant quantity is the outer probability from Section~\ref{sec:prelims}: probability over the random labeled training dataset used to construct the interval bounds. Our goal is to show that, with probability at least $\lambda$, the admissible set of perception uncertainty models contains the true perception uncertainty model $Z^\ast$ entry-wise.
Let
\[
\mathcal{E}_{\mathrm{corr}}
=
\bigcap_{s\in S,\;o\in O}
 Z^\ast(o\mid s)\in [Z^-(o\mid s),Z^+(o\mid s)] 
\]
denote the event that every entry of the true perception uncertainty model lies inside its corresponding learned confidence interval.
By union bound, if the entry-wise failure probabilities satisfy
\[
\sum_{s\in S}\sum_{o\in O}\alpha_{s,o}\le\alpha,
\]
where $\alpha_{s,o}$ is the probability that the confidence interval for entry $(s,o)$ fails to contain the true value $Z^\ast(o\mid s)$,
then
		\[
		\Pr(\mathcal{E}_{\mathrm{corr}})\ge 1-\alpha \triangleq \lambda.
		\]
		In words, $\lambda$ is the probability, over the random training dataset used to construct the intervals, that the admissible set of perception uncertainty models $\mathcal{Z}$ contains the true perception uncertainty model $Z^\ast$. This is the dataset-level correctness probability that appears in Theorem~\ref{thm:outer-pcis-safety}.

	\paragraph*{Why the union bound is conservative}
			The union bound makes no structural assumptions about the observation process: it holds even when the entry-wise correctness events $\{Z^\ast(o\mid s)\in[Z^-(o\mid s),Z^+(o\mid s)]\}$ are arbitrarily dependent. In practice, these events are typically correlated. For a fixed true state $s$, all counts $\{k_{s,o}\}_{o\in O}$ are computed from the same pool of $n_s$ samples, so increasing the empirical mass assigned to one observation necessarily reduces the mass available to others. Because of this coupling across observations within the same state, summing the per-entry failure probabilities can substantially under-estimate $\Pr(\mathcal{E}_{\mathrm{corr}})$.

If additional assumptions are defensible, this outer bound can be tightened. For example, under independence of the per-entry failure events one has
	\[
	\Pr(\mathcal{E}_{\mathrm{corr}})\ge \prod_{s\in S}\prod_{o\in O}(1-\alpha_{s,o}),
	\]
		which allows selecting larger per-entry $\alpha_{s,o}$ for the same target $\lambda$ than the simple constraint $\sum_{s,o}\alpha_{s,o}\le\alpha$. More generally, any joint (simultaneous) confidence region construction that exploits structure---such as row-wise multinomial constraints, state-wise independence, or parametric assumptions on the perception uncertainty model---can replace the Bonferroni/union-bound.

\section{Guarantees}
\label{sec:guarantees}
This section gives two guarantees. The first is a conditional abstraction-soundness statement: if the true perception uncertainty model lies in the admissible set of perception uncertainty models, then any action admitted by the deployed shield is guaranteed to be $\Omega$-admissible with probability at least $\beta$. This follows from the construction of the belief envelope, which maintains $\mathcal{B}_t\subseteq\widehat{\mathcal{B}}_t$. The second adds the dataset-level correctness probability from Section~\ref{sec:interval-construction-and-pac}: when $\Omega$ is instantiated from a \PCIS-style safe subset $C\subseteq S$, we combine abstraction soundness with the probability $\lambda$ that the admissible set of perception uncertainty models contains the true perception uncertainty model to obtain a finite-horizon safety guarantee.

\begin{theorem}[Abstraction Soundness]
\label{thm:soundness}
Assume $Z^\ast\in\mathcal{Z}$. Then for any $a\in\widehat{\Xi}(h_t)$ and the true state ($s_t$) at time $t$,
\[
\inf_{Z\in\mathcal{Z}}\Pr_Z[a\in\Omega(s_t)\mid h_t]\ge\beta.
\]
\end{theorem}
\paragraph*{Proof sketch.}
By construction of the propagated envelope, $\mathcal{B}_t\subseteq\widehat{\mathcal{B}}_t$ for all $t$. Therefore, for each perception uncertainty model $Z\in\mathcal{Z}$, the history $h_t$ induces a posterior belief $b_t^Z\in\mathcal{B}_t\subseteq\widehat{\mathcal{B}}_t$, and the probability that $a$ is $\Omega$-admissible under that posterior is exactly the linear score $\phi_a(b_t^Z)$. Hence the worst-case conditional admissibility probability over $Z\in\mathcal{Z}$ is lower bounded by the minimum of $\phi_a$ over the envelope, namely $\underline{p}_t(a)$, which is at least $\beta$ whenever $a\in\widehat{\Xi}(h_t)$. The full proof appears in the \supplementname.
In words, any action admitted by the deployed shield is guaranteed to be $\Omega$-admissible with probability at least $\beta$ for every perception uncertainty model in the admissible set.

\paragraph*{Conservativeness.}
If $a\notin\widehat{\Xi}(h_t)$, then either $a\notin\Xi^\star(h_t)$, or $a$ is excluded due to envelope/relaxation over-approximation.

Thm \ref{thm:soundness} and \textit{Conservativeness} hold for any perfect-perception shield $\Omega$. To obtain an explicit finite-horizon safety lower bound, we instantiate $\Omega$ with a \PCIS-style invariance condition.

\begin{assumption}[PCIS Instantiation: One-Step Bound]
\label{assump:pcis-one-step}
We instantiate the perfect-perception shield $\Omega$ using the one-step \PCIS-style admissibility constraint in Equation~\eqref{eq:omega-pcis}. Assume that there exists a designated safe subset $C\subseteq S$ and a constant $\gamma\in(0,1]$ such that for all $s\in C$ and $a\in\Omega(s)$,
\[
\sum_{s'\in C}T(s,a,s')\ge\gamma.
\]
Thus, whenever the current state lies in $C$ and an action admitted by $\Omega$ is taken, the next state remains in $C$ with probability at least $\gamma$.
\end{assumption}

\begin{theorem}[Finite-Horizon Double Probability]
\label{thm:outer-pcis-safety}
Let $D$ denote the random training dataset, and let $\widehat{\Xi}^D$ and $\mathcal{Z}^D$ be the resulting deployed shield and admissible set of perception uncertainty models. Assume that the support of the initial belief $b_0$ is contained in $C$, actions are chosen from $\widehat{\Xi}^D$, and Assumption~\ref{assump:pcis-one-step} holds. Then for horizon $H$,
\[
\Pr_D\!\left[
\Pr\!\left[\mathrm{Safe}_{0:H}\right]\ge (\beta\gamma)^H
\right]\ge \lambda.
\]
\end{theorem}
In words, the outer probability is over the randomness in the training dataset used to construct the interval bounds. With probability at least $\lambda$ over that dataset randomness, the resulting deployed shield guarantees closed-loop safety over horizon $H$ with probability at least $(\beta\gamma)^H$. Here $\mathrm{Safe}_{0:H}$ denotes the event that the trajectory starts in the safe subset $C$ and remains in $C$ through time $H$. (This is a finite-horizon guarantee, and the lower bound can decay quickly as $H$ grows.)

\paragraph*{Proof sketch.}
Condition on the model-correctness event that the true perception uncertainty model lies in the admissible set of perception uncertainty models. Along any safe prefix, Abstraction Soundness guarantees that the deployed shield selects an $\Omega$-admissible action with probability at least $\beta$, while Assumption~\ref{assump:pcis-one-step} gives probability at least $\gamma$ of remaining in the safe subset $C$ once such an action is taken. Therefore each step preserves safety with probability at least $\beta\gamma$, and a chain-rule argument over $H$ steps yields the conditional lower bound $(\beta\gamma)^H$ on $\Pr[\mathrm{Safe}_{0:H}]$. Finally, Section~\ref{sec:interval-construction-and-pac} gives the outer probability bound $\lambda$ for the model-correctness event, producing the stated double-probability guarantee. The full proof appears in the \supplementname.

\rev{Thus the theorem should be read as a finite-horizon design certificate whose usefulness depends on the design point. Since the lower bound scales as $(\beta\gamma)^H$, the threshold $\beta$, the one-step invariance level $\gamma$, and the horizon $H$ must be chosen carefully for the guarantee to remain meaningful; for lower per-step values or longer horizons, the formal lower bound can become vacuous even when empirical safety remains high. Section~\ref{sec:evaluation} therefore compares the numerical certificate values with empirical rollout safety to show when the bound is informative.}

\paragraph*{\rev{State quantization}}
\rev{The theorems above are stated for a finite \IPOMDP. If the original system has continuous state or action variables, a separate abstraction step is needed before applying the shield. Existing stochastic-abstraction methods can provide probabilistic error bounds between a continuous stochastic system and a finite abstraction \cite{soudjani2013adaptive,lavaei2022automated}. Such bounds can in principle be composed with our interval-perception guarantee by treating the abstraction error as an additional dynamics-side error term in the construction of the safe set or in the one-step lower bound $\gamma$. Developing this composition for a particular continuous abstraction method is outside the scope of the present paper; here our guarantees apply to the finite models used by the shield.}

\section{Experimental Evaluation}
\label{sec:evaluation}
\subsection{Setup}
\label{sec:evaluation-setup}
We evaluate five shielding methods on four benchmark autonomous systems, usually modeled as \MDP or \POMDP, but here represented as \IPOMDPs. These benchmarks are chosen to vary the ratio of states ($S$) to observations ($O$) and the scale of the underlying decision problem: TaxiNet (16 states, 16 observations), Obstacle (50 states, 3 observations), CartPole (82 states, 82 observations), and Refuel (344 states, 29 observations). TaxiNet is a runway-alignment benchmark in which safety means staying within the cross-track and heading-error bounds. Obstacle is a grid-navigation benchmark in which safety means avoiding obstacle cells. CartPole is a balance-control benchmark in which safety means remaining within the standard CartPole position and pole-angle limits. Refuel is a gridworld with fuel constraints in which safety means avoiding both obstacle collision and fuel exhaustion away from a refuel station.

Obstacle and Refuel are adapted from prior support-shielding case studies \cite{carr2023pomdpshielding}, but their observation kernels are modified here to fit the interval-perception setting and allow us to study the effect of shielding. In particular, the original Refuel benchmark could be solved nearly perfectly by an RL-trained controller, leaving little room for shielding to improve safety. We therefore modify Refuel to remove direct observation of the crash predicate and of whether fuel is nonzero, so the safety-critical variables are latent and must be inferred from history rather than read directly from the current observation. TaxiNet uses observation uncertainty derived from real autonomous-taxiing perception-confusion data \cite{xtaxinet}\rev{\cite{fremont2020formal,kadron2022case}}. CartPole likewise uses a learned perception model, but in this case we train the perception model ourselves for the benchmark. Obstacle and Refuel instead use synthetic observation uncertainty, meaning that we begin from a hand-specified observation function and perturb it with an observation-noise budget to obtain interval-valued observation probabilities. These domains stress different partial-observability regimes, from near-bijective observation kernels (CartPole) to severe compression of many states into a few observations (Obstacle) and latent safety-critical variables that must be inferred from history (Refuel). In each case, the dynamics are fixed and the observation kernel uses confidence intervals, yielding the \IPOMDP used in evaluation\rev{; unless otherwise noted, these are Clopper--Pearson intervals at significance level $\alpha=0.05$}.

\rev{The continuous-state benchmarks are first converted to finite models before applying our shielding construction. TaxiNet uses the TaxiNet abstraction from prior work \cite{xtaxinet,fremont2020formal,kadron2022case,pasareanu2023closedloop}: cross-track error is discretized to \(-2,\ldots,2\), heading error to \(-1,0,1\), and an absorbing failure state gives 16 states. CartPole uses three uniform bins for each of position, velocity, pole angle, and angular velocity over the artifact ranges \([-0.318,0.397]\), \([-2.288,2.550]\), \([-0.258,0.266]\), and \([-3.094,3.290]\), giving \(3^4+1=82\) total states. Obstacle and Refuel are finite gridworlds: Obstacle uses a \(7\times 7\) grid plus failure, and Refuel uses a \(7\times 7\) grid with fuel levels \(0,\ldots,6\) plus failure. The induced \PCIS one-step bounds are \(\gamma=1\) for all four benchmarks: the PCIS shields admit only actions whose dynamics support remains safe in the finite models used for evaluation.}

\orangechange{We compare five shielding methods by measuring failure, stuck, and safe rates under the same RL-trained controller and perception regime; additional evaluation details appear in the \supplementname.}

\subsection{Baselines}
\label{sec:evaluation-baselines}
The comparison includes two prior reference points and three methods introduced in this paper. The prior baselines are the \emph{Observation} shield and the support-based shield of Carr \emph{et al.} \cite{carr2023pomdpshielding}. By contrast, \emph{Single-Belief}, \emph{Fwd-Sampling}, and \emph{Envelope} are all novel methods in this work: Single-Belief is our point-estimate history-based shield, Fwd-Sampling is our sampled under-approximation to reachable interval beliefs, and Envelope is our LP-based over-approximation shield.

\orangechange{The \emph{Observation} shield is memoryless: at each step it computes a posterior distribution over states using only the current observation, a uniform prior, and the point-estimate perception model, and it allows exactly those actions whose probability of being safe under that posterior exceeds the threshold. The point-estimate perception model is built from expected emission probabilities. The \emph{Single-Belief} shield maintains a standard Bayesian belief under point-estimate observations and applies the same threshold test to the current belief. The proposed \emph{Envelope} shield maintains a template polytope over reachable beliefs under interval uncertainty and admits an action only when its worst-case safety score exceeds the threshold. The \emph{Fwd-Sampling} shield instead tracks an under-approximation of the reachable belief set using forward-sampled concrete belief points; it maintains a budget of \(N=500\) belief points and uses \(K=100\) sampled observation-probability vectors per propagation step. Finally, we include the support-based shield of Carr \emph{et al.} \cite{carr2023pomdpshielding}, computed on the point-estimate POMDP whenever the support-MDP construction is feasible. That construction builds an MDP whose states are reachable belief supports and computes a winning region offline.}

Envelope propagation is practical for TaxiNet and Obstacle, but not for CartPole and Refuel. Fwd-Sampling is feasible on all four benchmarks. Carr is feasible for TaxiNet, Obstacle, and CartPole but infeasible on Refuel because the number of reachable supports in the support-MDP becomes too large during the offline breadth-first exploration, and on TaxiNet it degenerates because the point-estimate POMDP has no winning support.

\orangechange{This feasibility split shapes the interpretation. Envelope is the strictest and most expensive method because it certifies actions against an over-approximation of all reachable beliefs under interval uncertainty. We therefore use it as the strongest robust reference point when tractable, while Fwd-Sampling is the scalable approximation at much lower online cost. Evaluating both separates what stricter interval-belief reasoning buys from how closely cheaper sampling recovers it in practice.}

\subsection{Evaluation Protocol}
\label{sec:evaluation-protocol}
\orangechange{To evaluate each shield, we run Monte Carlo rollouts of the closed-loop system consisting of the \IPOMDP, an RL-trained controller, and the shield. In each rollout, the current history is passed to the controller, the shield filters the controller's proposed action, the next latent state is sampled from the dynamics, and the next observation is sampled from the perception mechanism. Because the observation probabilities are interval-valued, we consider two ways of instantiating perception during evaluation. In the \emph{uniform} regime, each rollout samples a valid observation kernel within the interval constraints and then samples observations from that kernel throughout the rollout. In the \emph{adversarial} regime, we first solve an offline optimization problem over fixed interval realizations and then hold the resulting realization fixed for the entire rollout. Concretely, we optimize over observation kernels within the interval bounds using a cross-entropy method to maximize the empirical failure rate against the deployed shield-compliant RL controller.} \rev{We do not use a time-varying adversary that chooses a new observation kernel at each step: such an adversary would match the worst-case variability already allowed by our interval-belief shields and would therefore make these interval-belief shields appear stronger relative to evaluations under a fixed-kernel perception model.}

\orangechange{All reported numbers use an RL-trained controller from our experimental implementation. The controller is a Deep Q-Network trained directly on the \IPOMDP dynamics with observation-action history as input, using a safety-oriented reward that penalizes reaching \texttt{FAIL}, rewards completing the finite horizon safely, and gives a small per-step reward for survival. The controller proposes an action and the shield either admits it or replaces it by a random admissible action; details appear in the \supplementname. We sweep $\beta\in\{0.50,0.60,0.65,0.70,0.75,0.80,0.85,0.90,0.95\}$ for all threshold-based shields, run 200 Monte Carlo rollouts per condition, and report three mutually exclusive outcomes: \emph{fail} (the trajectory reaches a \texttt{FAIL} state),
\emph{stuck} (the trajectory reaches a state where the shield blocks every action),
and \emph{safe} (neither event occurs within the benchmark horizon). The benchmark horizons are 20 steps for TaxiNet, 25 for Obstacle, 15 for CartPole, and 30 for Refuel.}

For threshold-based shields we report the operating point that minimizes fail rate, with stuck rate as a tiebreaker, and we use the per-threshold sweeps to discuss the corresponding stuck-avoidance tradeoffs.

\subsection{Cross-Case Picture}
\begin{figure*}[t]
\centering
\includegraphics[width=0.92\textwidth]{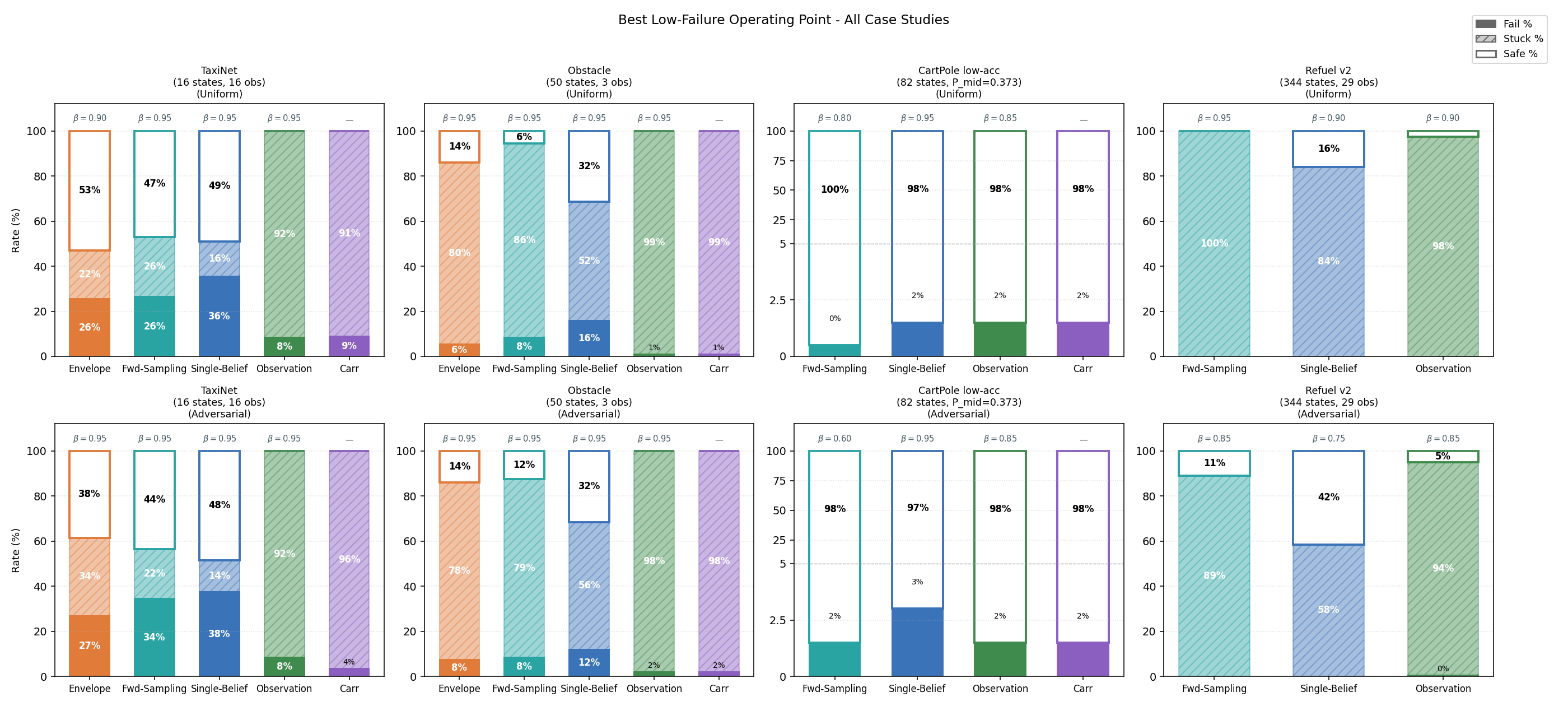}\caption{\orangechange{Best low-failure operating point of each shield under uniform and adversarial perception. The selected threshold minimizes fail rate and then stuck rate. Observation informativeness determines which shield family is useful, and adversarial harshness varies by case study. Only CartPole uses a disjoint piecewise vertical scale, expanding $0$--$5\%$ and compressing $5$--$100\%$; other case studies use the standard scale. The corresponding numerical values and uncertainty intervals appear in Table~S-I.}}
\label{fig:summary-bars}
\end{figure*}

\begin{figure*}[t]
\centering
\includegraphics[width=0.92\textwidth]{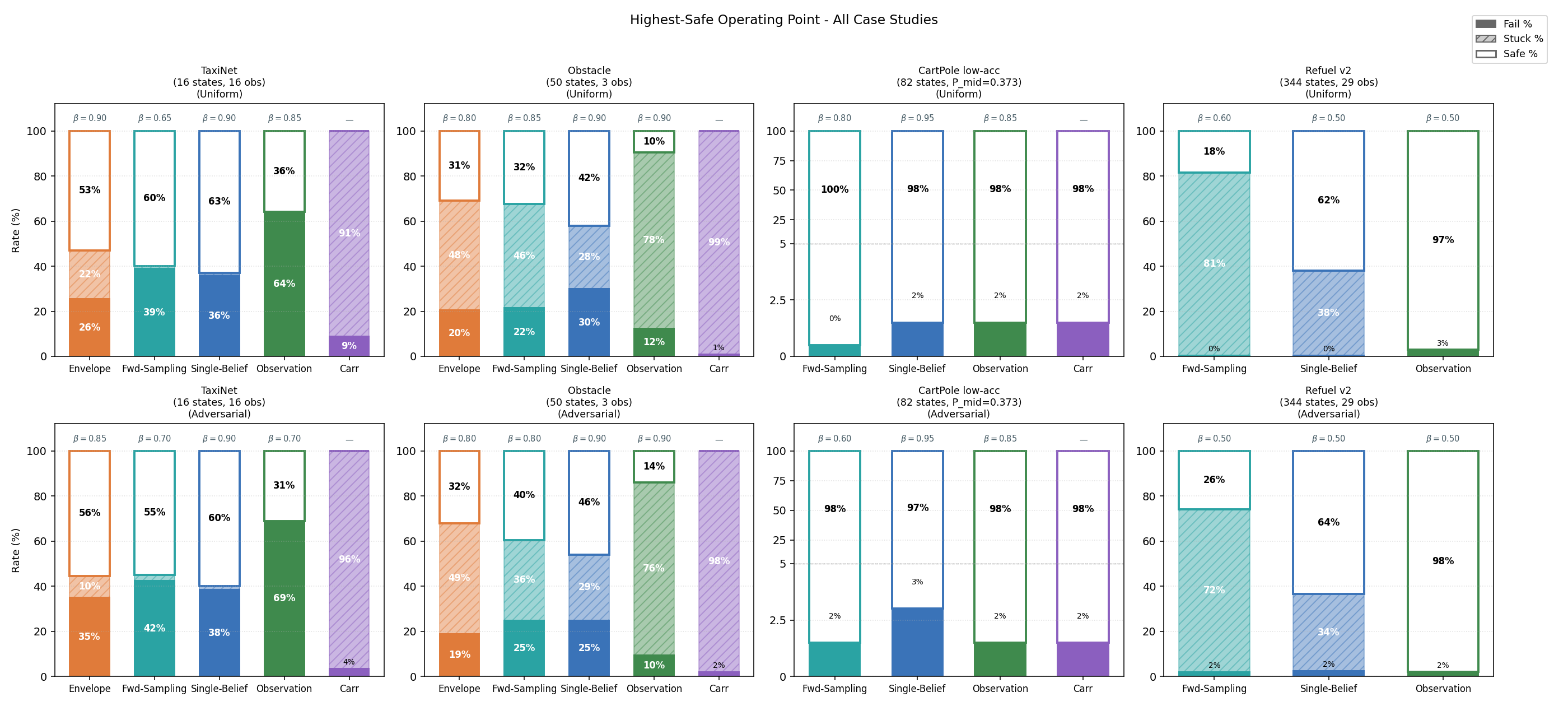}
\caption{\orangechange{Highest-safe operating point of each shield. The selected threshold maximizes safe completions, with lower fail and then lower stuck as tiebreakers. Solid color denotes fail, hatching denotes stuck, and the outlined top segment denotes safe. Threshold labels match Fig.~\ref{fig:summary-bars}. Only CartPole uses the disjoint piecewise vertical scale; other case studies use the standard scale. The corresponding numerical values and uncertainty intervals appear in Table~S-II.}}
\label{fig:summary-bars-safe}
\end{figure*}

\orangechange{Figure~\ref{fig:summary-bars} summarizes the lowest-failure setting across the four benchmarks. A key driver of the observed differences appears to be observation informativeness: how strongly the current observation narrows down the latent state. When observations almost identify the latent state, as in CartPole, all feasible shields are nearly equivalent and incur essentially no stuck-rate penalty. At the opposite extreme, Obstacle compresses 50 states into only 3 observations; memoryless methods then lose any useful middle ground between allowing risky actions and blocking almost everything, while belief history appears necessary. TaxiNet sits between these extremes: it has as many observation labels as states, but the learned perception model is noisy enough that one-step posteriors remain ambiguous. Refuel is different again: the observation kernel hides the variables that determine whether collision or fuel exhaustion is imminent, so history is useful but even accurate belief tracking can lead to paralysis.}

\orangechange{The proposed envelope construction improves the safety side of the tradeoff whenever it is feasible to run. Here each shield is evaluated at its own threshold selected by the low-failure rule of Fig.~\ref{fig:summary-bars}, so the comparison is between each method's preferred safety-first operating point rather than at a common shared threshold. On TaxiNet, Envelope remains the strongest robust option, reaching $25.5\%$ fail and $21.5\%$ stuck under uniform perception, compared with $35.5\%$ fail and $15.5\%$ stuck for Single-Belief and $26.5\%$ fail and $26.5\%$ stuck for Fwd-Sampling; under adversarial perception the same ordering remains, at $27.0\%/34.5\%$ for Envelope versus $37.5\%/14.0\%$ for Single-Belief and $34.5\%/22.0\%$ for Fwd-Sampling. On Obstacle, the gain is larger: Envelope reaches $6\%$ fail and $80\%$ stuck under uniform perception, while Single-Belief remains at $16\%$ fail and $52\%$ stuck and Fwd-Sampling improves to $8\%$ fail and $86\%$ stuck. Under adversarial perception, Envelope reaches $8\%$ fail and $78\%$ stuck, still clearly improving on Single-Belief's $12\%/56\%$. This extra stuck behavior is the empirical counterpart of the over-approximation soundness result: because the envelope certifies actions against all beliefs in an over-approximation of the reachable set, it can block actions that look safe under favorable point-estimate beliefs but are not safe in the worst case.}

\orangechange{The adversarial columns show that worst-case observation realizations affect benchmarks unevenly. They are harsher on history-based shields in TaxiNet and, less so, Obstacle, but remain closer to uniform results in CartPole and Refuel. This pattern is consistent with TaxiNet and Obstacle leaving more room for an optimized observation kernel to interfere with the closed-loop policy.}

\orangechange{Figure~\ref{fig:summary-bars-safe} provides the complementary stuck-avoidance view. Once the threshold is selected to maximize safe completions rather than minimize failures, the ranking changes in the direction suggested by the low-failure figure. On TaxiNet and Obstacle, Single-Belief or Fwd-Sampling usually become preferable to Envelope, suggesting that the extra conservatism of Envelope can convert many potentially safe runs into stuck runs. On Refuel, the strongest stuck-avoidance picture is dominated by Observation, which reaches about $97\%$ safe under uniform perception and $98\%$ under adversarial perception at low thresholds, whereas the zero-failure operating points of Single-Belief and Fwd-Sampling have much higher stuck rates. CartPole remains near-degenerate across all shields, with every method achieving roughly $97$--$99\%$ safe completions and no stuck episodes at its best threshold.}

\orangechange{The remaining baselines illustrate complementary failure modes. The Observation shield can be competitive when a single observation already localizes the state, but it degrades sharply under aliasing, meaning that multiple latent states produce the same or very similar observations. In TaxiNet, its low-failure point ($8.5\%$ fail under uniform perception) is only obtained at $91.5\%$ stuck, which is still a much less attractive tradeoff than what the envelope achieves by using history quantitatively. In Obstacle, Observation becomes almost indistinguishable from support-only shielding, both ending at approximately $1$--$2\%$ fail and $98$--$99\%$ stuck. In Refuel, however, Observation has a different role: at lower thresholds it is still the only shield that attains a $0\%$-stuck operating point, albeit at $3$--$4.5\%$ failure. Fwd-Sampling sits between Single-Belief and Envelope on TaxiNet and Obstacle, but on Refuel it behaves much more conservatively, with $0\%$ fail at the cost of $100\%$ stuck under uniform perception and about $89\%$ stuck under adversarial perception. Carr's shield is strongest when support information is already nearly enough to identify the state, but it can collapse entirely once every reachable support contains conflicting safety requirements. TaxiNet is the clearest example: under Carr's offline support-MDP analysis on the point-estimate POMDP, the winning region is empty, so Carr effectively blocks almost all runs from the initial step.}

\begin{figure}[t]
\centering
\includegraphics[width=0.98\columnwidth]{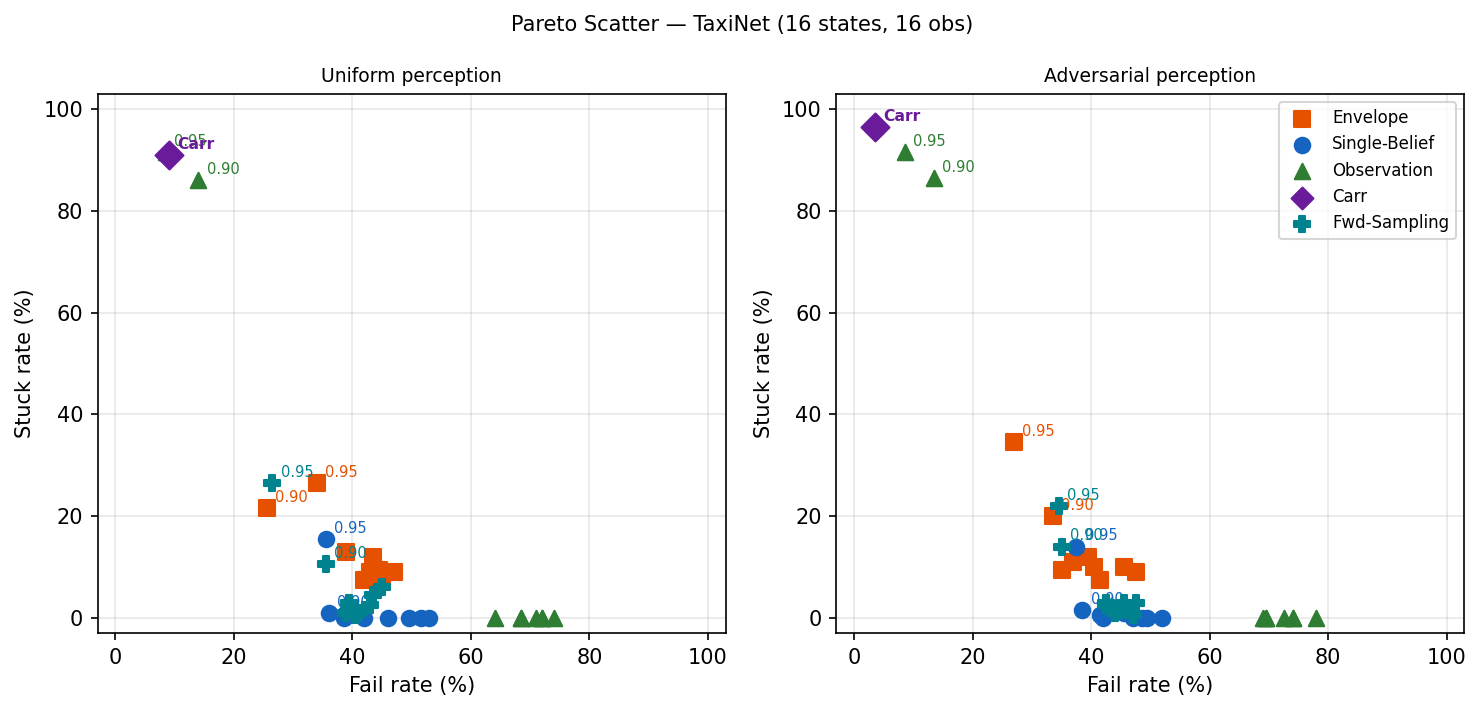}\caption{\orangechange{TaxiNet Pareto scatter over threshold $\beta$ under adversarial perception. History-based shields improve fail-versus-stuck tradeoffs over memoryless Observation, and Envelope improves on Single-Belief by removing actions safe only for the point-estimate model.}}
\label{fig:pareto-taxinet}
\end{figure}

\subsection{Comparison with Support-Based Shielding}
\label{sec:support-comparison}
The comparison with Carr-style support shielding clarifies when probability mass is important. Support reasoning can be highly effective when observations are informative enough that reachable supports stay small; the CartPole results are consistent with this, since Carr remains competitive there. But support-only reasoning discards the distinction between a state with posterior weight $0.49$ and one with posterior weight $10^{-4}$, and in aliased problems that distinction can determine whether a useful action should still be allowed. TaxiNet makes this failure mode stark: support shielding has no winning support at all, even though the envelope, Fwd-Sampling, and Single-Belief filters all retain nontrivial safe behavior. Obstacle shows the softer version of the same phenomenon, where support reasoning remains feasible but collapses to near-total stuck.

Taken together, these experiments support a narrow but important claim. Interval belief envelopes are not a universal replacement for simpler shields. Instead, they appear most useful in the regime where point-estimate belief filters are too optimistic and support abstractions are too coarse. CartPole suggests that they are unnecessary when observations already localize the state well. TaxiNet and Obstacle illustrate where they can be useful. Refuel shows the computational boundary where heavier robust shielding becomes impractical. Fwd-Sampling is a useful intermediate point in that design space: it is far cheaper than the LP envelope and available on all four benchmarks, but because it under-approximates the reachable belief set it generally trades extra blocked actions for only limited safety improvement. In the difficult aliased regimes studied here, the envelope provides the strongest safety performance among the feasible shields, and its extra conservatism can be quantified rather than treated as a black box.

\subsection{\texorpdfstring{\rev{Conformal Shielding Comparison}}{Conformal Shielding Comparison}}
\label{sec:conformal-comparison}
\rev{We also compare the interval-belief shields with the TaxiNet conformal shielding baseline of~\cite{ScarbroImrie25}. The comparison keeps the same finite TaxiNet \IPOMDP dynamics, action space, controller, and rollout protocol, but substitutes the simulated perception artifact from~\cite{ScarbroImrie25}. That artifact provides paired samples \((s,\hat{s},C_\kappa)\), where \(s\) is the true discretized state, \(\hat{s}\) is the point estimate consumed by our shields, and \(C_\kappa\) is the conformal prediction set consumed by the conformal shield. We model this as \(P(\hat{s},C_\kappa\mid s)=Z(\hat{s}\mid s)Q_\kappa(C_\kappa\mid s,\hat{s})\): the interval-belief shields use \(Z\), while the conformal shield samples its set conditionally on the realized \((s,\hat{s})\), including under adversarial point-estimate realizations. This preserves the empirical dependence between point estimates and conformal sets; the \supplementname\ gives the full construction.}

\begin{figure}[t]
\centering
\includegraphics[width=0.82\columnwidth]{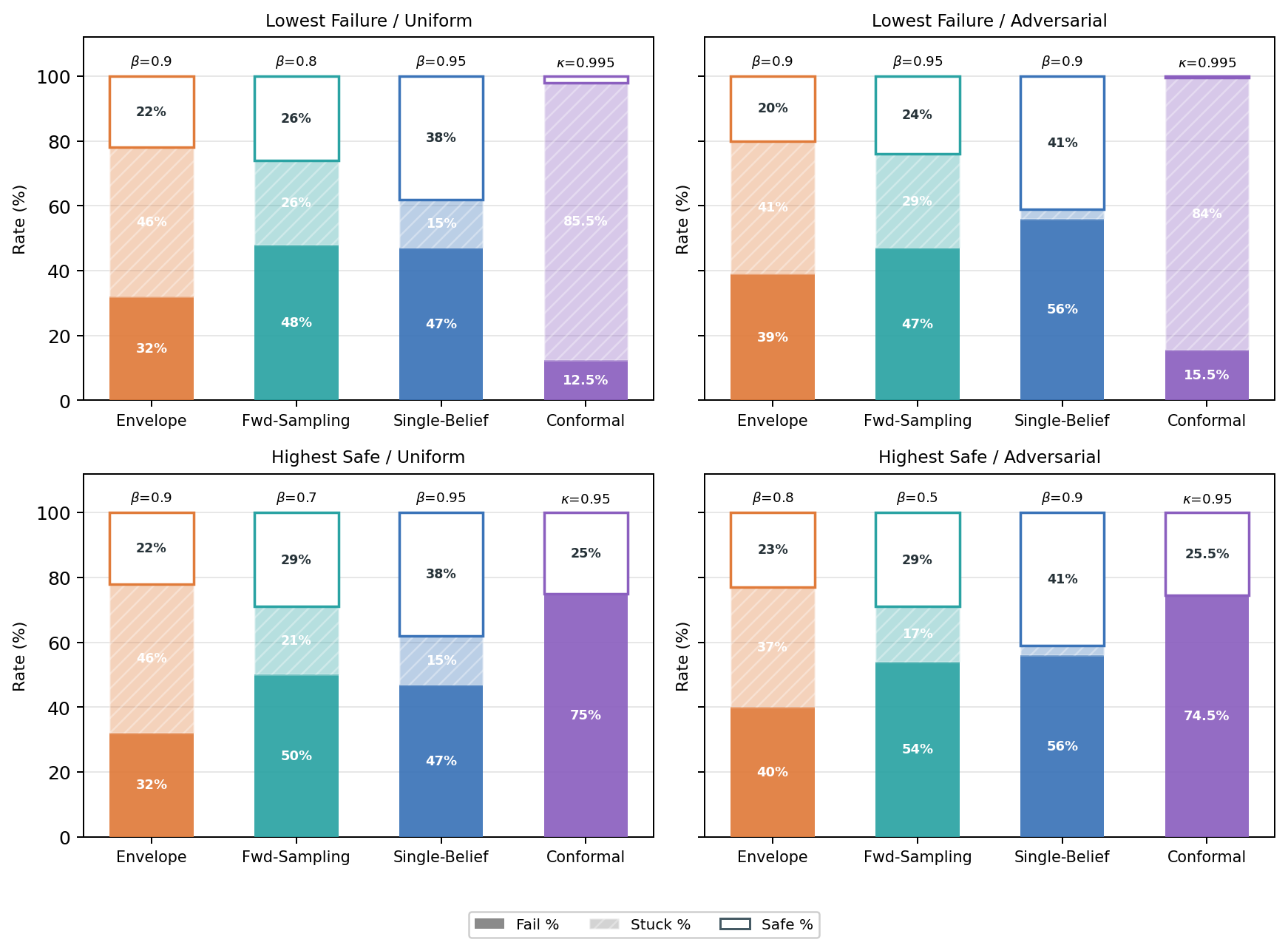}
	\caption{\rev{TaxiNet comparison under uniform and adversarial perception. The top row selects each method's lowest-failure operating point, while the bottom row selects each method's highest-safe operating point. The interval-belief shields (Single-Belief, Envelope, Fwd-Sampling) are swept over threshold \(\beta\); the Conformal baseline is swept over confidence level \(\kappa\). The setting that each method selects in each panel is annotated above its bar. The corresponding numerical values and uncertainty intervals appear in Table~S-III.}}
\label{fig:taxinet-v2-extremal-comparison}
\end{figure}

\rev{Figure~\ref{fig:taxinet-v2-extremal-comparison} suggests a distinct tradeoff for the conformal baseline. At the lowest-failure operating point, the high-confidence conformal shield has lower failure rates than the interval-belief shields, but most rollouts become stuck: \(12.5\%\) fail and \(85.5\%\) stuck under uniform perception, and \(15.5\%\) fail and \(84\%\) stuck under adversarial perception. The strongest interval-belief shield in this row is Envelope, with \(32\%/46\%\) fail/stuck under uniform perception and \(39\%/41\%\) under adversarial perception. At the highest-safe operating point, the conformal shield does outperform Envelope and Fwd-Sampling on safe completions, reaching \(25\%\) and \(25.5\%\) in the two regimes. However, this operating point still leaves about three quarters of rollouts failing, while Single-Belief reaches \(38\%\) and \(41\%\) safe completions. In this experiment, the conformal baseline appears to reduce failures mainly by blocking most rollouts at high confidence, while lower confidence improves liveness at the cost of a high failure rate. The interval-belief shields exhibit a different fail-versus-stuck tradeoff, consistent with their use of history and probability mass within the uncertainty model.}

\rev{This experiment is empirical and does not test the exchangeability assumption behind conformal coverage under closed-loop policy changes; it compares the two approaches assuming no distribution shift. The high stuck rates for conformal shielding at high confidence are consistent with a method that does not account for observation history or state probabilities. In other words, conformal shielding appears to be a coarser method that preserves less information about the state distribution and therefore makes more conservative shielding choices.}

\subsection{\texorpdfstring{\rev{Additional Analyses}}{Additional Analyses}}
\rev{Table~\ref{tab:bgamma-calibration} compares the finite-horizon certificate with empirical rollout outcomes for the tractable Envelope cases. The comparison should be read as a lower-bound sanity check, not as a prediction of rollout rates: the theorem certifies remaining in the safe state set, whereas the plotted safe-completion outcome excludes both failures and stuck rollouts. For that reason the table compares the certificate with empirical non-failure rather than with the safe-completion segment of the stacked bars. The certificate is well below the observed non-failure rates across the tested thresholds, as expected, because the proof compounds a worst-case one-step lower bound over the full horizon.}

\begin{table*}[t]
\centering
\rev{
		\caption{\rev{Finite-horizon certificate values compared with empirical safety for the Envelope shield at three runtime thresholds. Uncertainty is one binomial standard deviation over 200 rollouts. Empirical non-failure, \(1-\)fail, is the closest analogue of \(\mathrm{Safe}_{0:H}\), because stuck rollouts are liveness failures but need not leave the safe state set. Since the induced \PCIS one-step bound is \(\gamma=1\) for these benchmark models, the certificate column reports \(\beta^H\).}}
\label{tab:bgamma-calibration}
\scriptsize
\begin{tabular}{llccccccc}
\toprule
Case study & Regime & \(\beta\) & \(H\) & Fail & Stuck & Emp. non-failure & \(\beta^H\) & Gap\\
\midrule
TaxiNet & uniform & 0.85 & 20 & \(0.435\!\pm\!0.035\) & \(0.120\!\pm\!0.023\) & \(0.565\!\pm\!0.035\) & 0.039 & \(0.526\!\pm\!0.035\)\\
TaxiNet & uniform & 0.90 & 20 & \(0.255\!\pm\!0.031\) & \(0.215\!\pm\!0.029\) & \(0.745\!\pm\!0.031\) & 0.122 & \(0.623\!\pm\!0.031\)\\
TaxiNet & uniform & 0.95 & 20 & \(0.340\!\pm\!0.033\) & \(0.265\!\pm\!0.031\) & \(0.660\!\pm\!0.033\) & 0.358 & \(0.302\!\pm\!0.033\)\\
TaxiNet & adversarial & 0.85 & 20 & \(0.350\!\pm\!0.034\) & \(0.095\!\pm\!0.021\) & \(0.650\!\pm\!0.034\) & 0.039 & \(0.611\!\pm\!0.034\)\\
TaxiNet & adversarial & 0.90 & 20 & \(0.335\!\pm\!0.033\) & \(0.200\!\pm\!0.028\) & \(0.665\!\pm\!0.033\) & 0.122 & \(0.543\!\pm\!0.033\)\\
TaxiNet & adversarial & 0.95 & 20 & \(0.270\!\pm\!0.031\) & \(0.345\!\pm\!0.034\) & \(0.730\!\pm\!0.031\) & 0.358 & \(0.372\!\pm\!0.031\)\\
Obstacle & uniform & 0.85 & 25 & \(0.13\!\pm\!0.023\) & \(0.65\!\pm\!0.034\) & \(0.88\!\pm\!0.023\) & 0.017 & \(0.858\!\pm\!0.023\)\\
Obstacle & uniform & 0.90 & 25 & \(0.12\!\pm\!0.023\) & \(0.71\!\pm\!0.032\) & \(0.88\!\pm\!0.023\) & 0.072 & \(0.808\!\pm\!0.023\)\\
Obstacle & uniform & 0.95 & 25 & \(0.06\!\pm\!0.016\) & \(0.81\!\pm\!0.028\) & \(0.94\!\pm\!0.016\) & 0.277 & \(0.668\!\pm\!0.016\)\\
Obstacle & adversarial & 0.85 & 25 & \(0.19\!\pm\!0.027\) & \(0.60\!\pm\!0.035\) & \(0.82\!\pm\!0.027\) & 0.017 & \(0.798\!\pm\!0.027\)\\
Obstacle & adversarial & 0.90 & 25 & \(0.11\!\pm\!0.022\) & \(0.71\!\pm\!0.032\) & \(0.90\!\pm\!0.022\) & 0.072 & \(0.823\!\pm\!0.022\)\\
Obstacle & adversarial & 0.95 & 25 & \(0.08\!\pm\!0.019\) & \(0.79\!\pm\!0.029\) & \(0.92\!\pm\!0.019\) & 0.277 & \(0.648\!\pm\!0.019\)\\
\bottomrule
\end{tabular}}
\end{table*}

\rev{The \supplementname\ reports two further diagnostics. First, a TaxiNet sweep over the Clopper--Pearson perception-interval significance level \(\alpha\in\{0.01,0.025,0.05,0.075,0.10,0.15,0.20,0.30\}\) shows that \(\beta\) is the dominant tuning parameter and that \(\alpha\) mainly fine-tunes interval conservatism within the operating regime selected by \(\beta\). Second, an envelope-conservatism study compares the LFP envelope with forward-sampled under-approximations under fixed-kernel and time-varying perception semantics, separating the conservatism due to belief-set over-approximation from the additional pessimism of protecting against time-varying \(Z_t\).}

\subsection{Inference-Time Tradeoffs}
\rev{Mean shield inference time follows the expected hierarchy: Observation, Single-Belief, and Carr run in the microsecond regime, Fwd-Sampling in the millisecond-to-subsecond regime, and LP-based Envelope two to five orders of magnitude slower where feasible, with Refuel the scalability boundary. Single-Belief is thus the strongest lightweight default, Fwd-Sampling a scalable middle ground, and Envelope a low-frequency or smaller-model mechanism for hard aliased regimes. Per-benchmark timings and further analysis of these results are available in Appendix~\ref{sec:inference-timing}.}

\section{Related Work}
The closest related line of work is shielding under partial observability. Carr \emph{et al.} \cite{carr2023pomdpshielding} compute the support-based shield over belief supports, retaining only which latent states remain possible after an observation history and discarding emission probabilities. That support-based shield yields strong soundness guarantees for avoid and reach-avoid objectives and is an important baseline for our setting. Our contribution differs in the semantic object being propagated: instead of reachable supports, we propagate quantitative belief envelopes induced by interval-valued emission probabilities. This retains information that the support-based shield intentionally throws away, which is useful when support winning regions collapse or become too coarse to admit a usable shield.

  Related work also abstracts learned perception by a compact probabilistic interface between latent state     and downstream symbolic reasoning. P\u{a}s\u{a}reanu \emph{et al.} \cite{pasareanu2023closedloop} replace a  vision stack by a confusion-matrix-derived stochastic abstraction to enable closed-loop probabilistic  analysis, while Cleaveland \emph{et al.} \cite{cleaveland2025conservative} construct conservative IMDP       perception abstractions from finite data using confidence intervals and intra-bin probability enlargement for probabilistic model checking. Sch\"afers \emph{et al.} \cite{schaefers2026perceptionbased} use learned perception uncertainty models to drive POMDP belief updates from visual observations. \rev{Peper \emph{et al.} \cite{peper2025unified} likewise pursue PAC-style finite-sample guarantees for vision-based autonomy, unifying probabilistic verification and validation of learned perception components.} Our work adopts the same modeling move of summarizing perception by state-conditional emission probabilities, but differs in both objective and semantics: rather than verifying an offline closed-loop stochastic model, we propagate interval-valued belief envelopes online and use them to enforce conservative shielding under partial observability.                                                                                             

\orangechange{Adjacent work on predictive uncertainty under shift clarifies why we use interval bounds on the perception uncertainty model rather than conformal wrappers. Conformal prediction gives finite-sample marginal coverage for labels or outputs under exchangeability \cite{vovk2005algorithmic,angelopoulos2023conformal}, with extensions for covariate shift \cite{tibshirani2019conformal,jonkers2024wcps}, online drift \cite{gibbs2021adaptive}, hidden Markov models, and off-policy MDP evaluation \cite{nettasinghe2023hmmconformal,foffano2023conformalope}. Thus conformal methods are not ruled out merely by control-based shift. However, their certified objects are prediction sets, predictive distributions, or policy-value intervals, not entry-wise bounds on the latent-state perception model $Z(o\mid s)$ needed for belief propagation under deployment shift. Work using conformal wrappers for action shielding \cite{ScarbroImrie25} could be improved for control-based shift, but would still struggle when deployment changes the conditional emission probabilities themselves. Our method isolates uncertainty at that conditional level, tying the guarantee to probable realizations of per-state emission probabilities rather than the aggregate deployment state distribution. Ovadia \emph{et al.} \cite{ovadia2019can} further motivate this caution by showing that point-valued uncertainty estimates can degrade under shift.}

\orangechange{More broadly, our work sits in the shielding literature \cite{alshiekh2018safe,jansen2020safe,konighofer2020shielding} and uses standard POMDP belief semantics \cite{kaelbling1998planning,pineau2003point}. It combines robust/imprecise probabilistic modeling \cite{givan2000bounded,nilim2005robust,zaffalon2009imprecise}, Clopper--Pearson finite-sample intervals \cite{clopper1934use}, and state-based shields such as \PCIS \cite{gao2021pcis} or barrier-style invariance methods \cite{ames2019control} into a runtime imperfect-perception shield built by lifting a perfect-perception shield through interval-valued belief envelopes.}

\section{Conclusion}
We presented a shielding method for imperfect-perception agents modeled as IPOMDPs
where the perception uncertainty model is learned from finite data and represented by confidence intervals rather than point estimates. The method lifts a perfect-perception shield into belief space by propagating conservative belief envelopes and admitting only actions whose worst-case shield score remains above threshold. For the finite discrete models considered here, this gives a runtime realization of a fundamentally history-based shield that scales much better than explicit history reasoning. It yields a finite-sample outer guarantee on model correctness together with a conditional lower bound on runtime safety, while preserving a modular interface to upstream state-based shield constructions such as \PCIS.
Detailed proofs, expanded benchmark and protocol descriptions, the complementary stuck-avoidance summary, and additional diagnostic evaluation appear in the \supplementname.

The empirical picture is that interval belief envelopes occupy a useful middle ground between point-estimate methods and the support-based shield. Relative to single-belief baselines, they provide more reliable action correction under perception uncertainty; relative to the support-based shield, they remain usable in regimes where support abstractions collapse or become too coarse. \rev{At the same time, the experiments make the price of this additional fidelity visible: Envelope propagation is computationally demanding, infeasible on the largest benchmark settings, and sensitive to the coarseness of the template abstraction used to represent reachable beliefs. Fwd-Sampling is a scalable alternative that mimics some of the conservatism of LP envelope shielding when Envelope propagation is unavailable, but it is an under-approximation and does not provide the sound over-approximation guarantee. The benchmark comparison sharpens where this extra machinery matters: CartPole shows that it is unnecessary when observations already localize the latent state well, TaxiNet and Obstacle show the partially observable regimes where sound envelope propagation improves the safety-usability tradeoff, and Refuel illustrates the scalability limit of the strongest robust method.}

\section*{Artifact Availability}
The research artifact---source code, models, and scripts to reproduce all experiments---is archived on Zenodo at \mbox{\url{https://doi.org/10.5281/zenodo.21539083}}.

\section*{Acknowledgment}
The generative AI tool OpenAI Codex was used in the preparation of this manuscript to improve the clarity of the writing and to prepare figures for data presentation.

\bibliographystyle{IEEEtran}
\bibliography{references}

\begin{IEEEbiography}[{\includegraphics[width=1in,height=1.25in,clip,keepaspectratio]{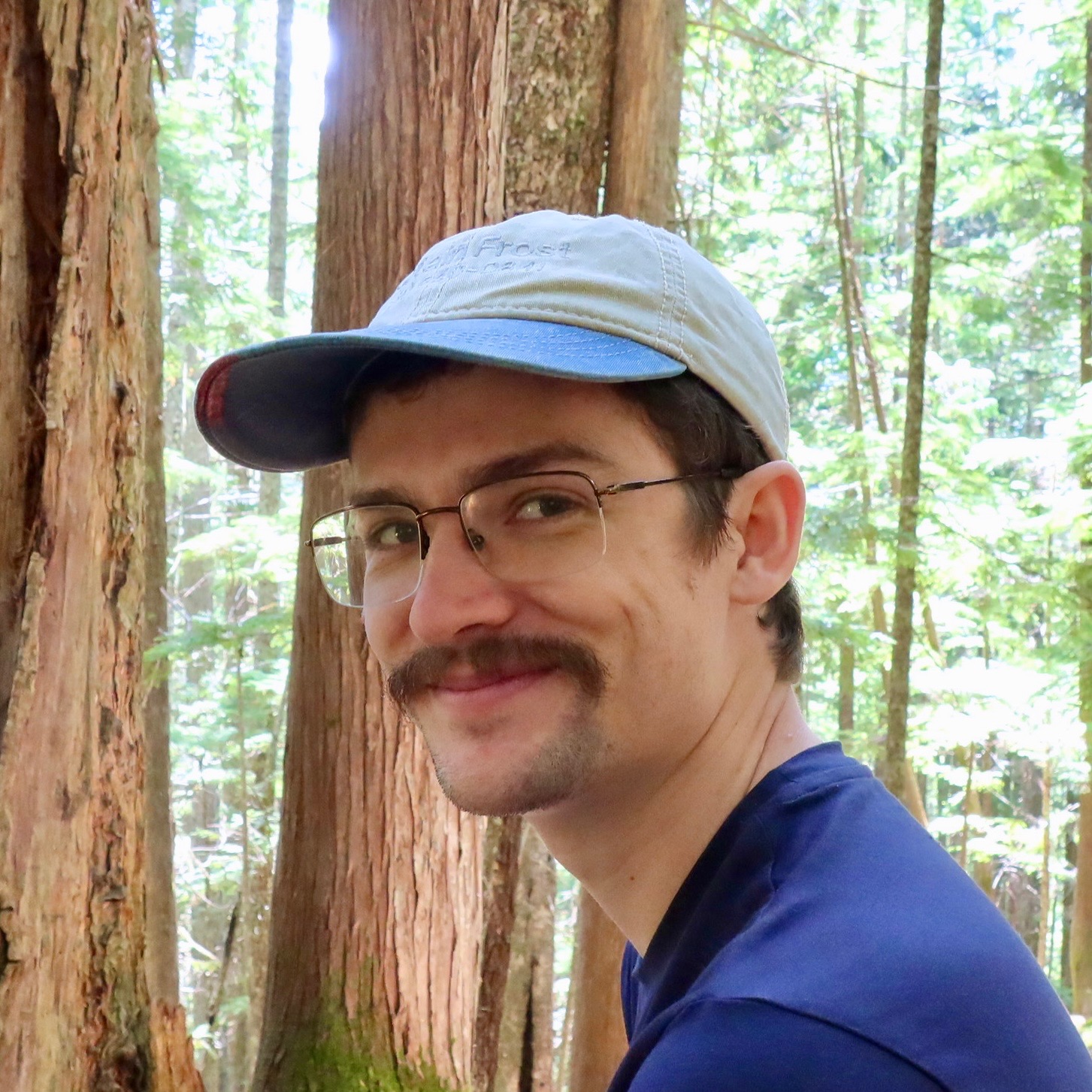}}]{William Scarbro}
completed the International Baccalaureate program in 2017. He received a B.S. in Computer Science and Applied Mathematics in 2021 and his M.S. in Computer Science from Colorado State University in 2024. William is currently a PhD candidate at Colorado State University using algebra to better understand AI-based control.
\end{IEEEbiography}

\begin{IEEEbiography}[{\includegraphics[width=1in,height=1.25in,clip,keepaspectratio]{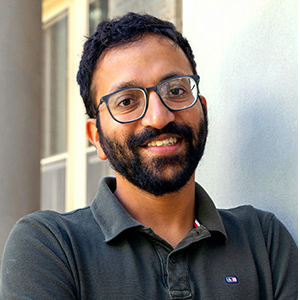}}]{Ravi Mangal}
received his PhD in Computer Science from the Georgia Institute of Technology in 2020 and was a Postdoctoral Researcher at CyLab, Carnegie Mellon University from 2021 to 2024. He is currently an Assistant Professor in the Department of Computer Science at Colorado State University. His research applies formal methods to the analysis of software systems with a focus on evaluating and improving the trustworthiness---safety, security, and interpretability---of AI-enabled systems.
\end{IEEEbiography}

\clearpage
\appendices
\setcounter{table}{0}
\renewcommand{\thetable}{S-\Roman{table}}
\section{\texorpdfstring{\rev{RL Controller and Adversarial-Perception Protocol}}{RL Controller and Adversarial-Perception Protocol}}
\label{sec:rl-controller-protocol}

\rev{The controller used in every Monte Carlo rollout in Section~VI is a Deep Q-Network trained directly on the \IPOMDP dynamics before shield evaluation, without a runtime shield in the loop. Its input is the last $W=10$ observation-action pairs, zero-padded when needed; flat tuple observations are encoded componentwise as numeric vectors, while non-tuple discrete observations and actions are one-hot encoded. The Q-network is a feedforward MLP with hidden layers $(128,64,32)$, ReLU activations, dropout $0.1$ after the first two hidden layers, and one output per action. Training uses replay capacity $50{,}000$, batch size $64$, Huber temporal-difference loss, Adam with learning rate $10^{-3}$, gradient clipping at $1.0$, discount factor $0.99$, and a target-network update every $500$ gradient steps. Exploration is $\epsilon$-greedy from $1.0$ with per-episode decay $0.995$ and floor $0.01$. Rewards are $+0.1$ per surviving step, $+10$ for completing the horizon without \texttt{FAIL}, and $-10$ for a \texttt{FAIL} transition. We use $500$ DQN training episodes for TaxiNet, Obstacle, low-accuracy CartPole, and the modified Refuel benchmark, and $300$ for the standard CartPole auxiliary runs.}

\rev{At evaluation time the DQN's exploration rate is zero. The deployed selector asks the DQN for its preferred action; if the shield admits it, the action is executed, and otherwise a uniformly random admissible action is executed. If no action is admissible, the rollout is recorded as stuck. The adversarial-perception regime is built with the same shield-compliant controller by optimizing one fixed observation-kernel realization inside the interval bounds. A candidate realization uses parameters $\alpha\in[0,1]^{|S|\times|O|}$ to interpolate between $\underline{Z}(o\mid s)$ and $\overline{Z}(o\mid s)$, followed by bounded-simplex projection for each row. We optimize these parameters with the Cross-Entropy Method. In each iteration, the optimizer samples a batch of candidate realizations from a clipped Gaussian, scores each candidate by Monte Carlo rollouts, and refits the Gaussian to the highest-failure candidates. For TaxiNet, Obstacle, and the modified Refuel benchmark, each iteration samples \(10\) candidate realizations and scores each candidate with \(5\) rollouts; we run \(10\) iterations, with initial standard deviation \(0.3\) and decay \(0.95\). CartPole uses the same \(5\) rollouts per candidate but \(8\) candidates and \(8\) iterations. The fixed realization is optimized against Envelope where feasible, then cached and reused across adversarial evaluation experiments.}

\section{Additional Empirical Results}

\begin{table*}[t]
\centering
\rev{\caption{Data underlying Fig.~1 of the main paper. Each entry reports the empirical percentage and a two-sided 95\% Clopper--Pearson interval over 200 rollouts for the corresponding marginal outcome.}
\label{tab:summary-lowfail-values}
\scriptsize
\begin{tabular}{lllcccc}
\toprule
Case study & Regime & Shield & Setting & Fail \% & Stuck \% & Safe \%\\
\midrule
TaxiNet & Uniform & Envelope & $\beta=0.90$ & $25.5$ [20,32] & $21.5$ [16,28] & $53$ [46,60]\\
TaxiNet & Uniform & Single-Belief & $\beta=0.95$ & $35.5$ [29,43] & $15.5$ [11,21] & $49$ [42,56]\\
TaxiNet & Uniform & Observation & $\beta=0.95$ & $8.5$ [5,13] & $91.5$ [87,95] & $0$ [0,2]\\
TaxiNet & Uniform & Carr & -- & $9$ [5,14] & $91$ [86,95] & $0$ [0,2]\\
TaxiNet & Uniform & Fwd-Sampling & $\beta=0.95$ & $26.5$ [21,33] & $26.5$ [21,33] & $47$ [40,54]\\
TaxiNet & Adversarial & Envelope & $\beta=0.95$ & $27$ [21,34] & $34.5$ [28,42] & $38.5$ [32,46]\\
TaxiNet & Adversarial & Single-Belief & $\beta=0.95$ & $37.5$ [31,45] & $14$ [10,20] & $48.5$ [41,56]\\
TaxiNet & Adversarial & Observation & $\beta=0.95$ & $8.5$ [5,13] & $91.5$ [87,95] & $0$ [0,2]\\
TaxiNet & Adversarial & Carr & -- & $3.5$ [1,7] & $96.5$ [93,99] & $0$ [0,2]\\
TaxiNet & Adversarial & Fwd-Sampling & $\beta=0.95$ & $34.5$ [28,42] & $22$ [16,28] & $43.5$ [37,51]\\
\addlinespace[1pt]
Obstacle & Uniform & Envelope & $\beta=0.95$ & $5.5$ [3,10] & $80.5$ [74,86] & $14$ [10,20]\\
Obstacle & Uniform & Single-Belief & $\beta=0.95$ & $16$ [11,22] & $52.5$ [45,60] & $31.5$ [25,38]\\
Obstacle & Uniform & Observation & $\beta=0.95$ & $1$ [0,4] & $99$ [96,100] & $0$ [0,2]\\
Obstacle & Uniform & Carr & -- & $1$ [0,4] & $99$ [96,100] & $0$ [0,2]\\
Obstacle & Uniform & Fwd-Sampling & $\beta=0.95$ & $8.5$ [5,13] & $86$ [80,90] & $5.5$ [3,10]\\
Obstacle & Adversarial & Envelope & $\beta=0.95$ & $7.5$ [4,12] & $78.5$ [72,84] & $14$ [10,20]\\
Obstacle & Adversarial & Single-Belief & $\beta=0.95$ & $12$ [8,17] & $56.5$ [49,63] & $31.5$ [25,38]\\
Obstacle & Adversarial & Observation & $\beta=0.95$ & $2$ [1,5] & $98$ [95,99] & $0$ [0,2]\\
Obstacle & Adversarial & Carr & -- & $2$ [1,5] & $98$ [95,99] & $0$ [0,2]\\
Obstacle & Adversarial & Fwd-Sampling & $\beta=0.95$ & $8.5$ [5,13] & $79$ [73,84] & $12.5$ [8,18]\\
\addlinespace[1pt]
CartPole & Uniform & Single-Belief & $\beta=0.95$ & $1.5$ [0,4] & $0$ [0,2] & $98.5$ [96,100]\\
CartPole & Uniform & Observation & $\beta=0.85$ & $1.5$ [0,4] & $0$ [0,2] & $98.5$ [96,100]\\
CartPole & Uniform & Carr & -- & $1.5$ [0,4] & $0$ [0,2] & $98.5$ [96,100]\\
CartPole & Uniform & Fwd-Sampling & $\beta=0.80$ & $0.5$ [0,3] & $0$ [0,2] & $99.5$ [97,100]\\
CartPole & Adversarial & Single-Belief & $\beta=0.95$ & $3$ [1,6] & $0$ [0,2] & $97$ [94,99]\\
CartPole & Adversarial & Observation & $\beta=0.85$ & $1.5$ [0,4] & $0$ [0,2] & $98.5$ [96,100]\\
CartPole & Adversarial & Carr & -- & $1.5$ [0,4] & $0$ [0,2] & $98.5$ [96,100]\\
CartPole & Adversarial & Fwd-Sampling & $\beta=0.60$ & $1.5$ [0,4] & $0$ [0,2] & $98.5$ [96,100]\\
\addlinespace[1pt]
Refuel & Uniform & Single-Belief & $\beta=0.90$ & $0$ [0,2] & $84$ [78,89] & $16$ [11,22]\\
Refuel & Uniform & Observation & $\beta=0.90$ & $0$ [0,2] & $97.5$ [94,99] & $2.5$ [1,6]\\
Refuel & Uniform & Fwd-Sampling & $\beta=0.95$ & $0$ [0,2] & $100$ [98,100] & $0$ [0,2]\\
Refuel & Adversarial & Single-Belief & $\beta=0.75$ & $0$ [0,2] & $58.5$ [51,65] & $41.5$ [35,49]\\
Refuel & Adversarial & Observation & $\beta=0.85$ & $0.5$ [0,3] & $94.5$ [90,97] & $5$ [2,9]\\
Refuel & Adversarial & Fwd-Sampling & $\beta=0.85$ & $0$ [0,2] & $89$ [84,93] & $11$ [7,16]\\
\bottomrule
\end{tabular}
}
\end{table*}

\begin{table*}[t]
\centering
\rev{\caption{Data underlying Fig.~2 of the main paper. Each entry reports the empirical percentage and a two-sided 95\% Clopper--Pearson interval over 200 rollouts for the corresponding marginal outcome.}
\label{tab:summary-highsafe-values}
\scriptsize
\begin{tabular}{lllcccc}
\toprule
Case study & Regime & Shield & Setting & Fail \% & Stuck \% & Safe \%\\
\midrule
TaxiNet & Uniform & Envelope & $\beta=0.90$ & $25.5$ [20,32] & $21.5$ [16,28] & $53$ [46,60]\\
TaxiNet & Uniform & Single-Belief & $\beta=0.90$ & $36$ [29,43] & $1$ [0,4] & $63$ [56,70]\\
TaxiNet & Uniform & Observation & $\beta=0.85$ & $64$ [57,71] & $0$ [0,2] & $36$ [29,43]\\
TaxiNet & Uniform & Carr & -- & $9$ [5,14] & $91$ [86,95] & $0$ [0,2]\\
TaxiNet & Uniform & Fwd-Sampling & $\beta=0.65$ & $39$ [32,46] & $1$ [0,4] & $60$ [53,67]\\
TaxiNet & Adversarial & Envelope & $\beta=0.85$ & $35$ [28,42] & $9.5$ [6,14] & $55.5$ [48,63]\\
TaxiNet & Adversarial & Single-Belief & $\beta=0.90$ & $38.5$ [32,46] & $1.5$ [0,4] & $60$ [53,67]\\
TaxiNet & Adversarial & Observation & $\beta=0.70$ & $69$ [62,75] & $0$ [0,2] & $31$ [25,38]\\
TaxiNet & Adversarial & Carr & -- & $3.5$ [1,7] & $96.5$ [93,99] & $0$ [0,2]\\
TaxiNet & Adversarial & Fwd-Sampling & $\beta=0.70$ & $42.5$ [36,50] & $2.5$ [1,6] & $55$ [48,62]\\
\addlinespace[1pt]
Obstacle & Uniform & Envelope & $\beta=0.80$ & $20.5$ [15,27] & $48.5$ [41,56] & $31$ [25,38]\\
Obstacle & Uniform & Single-Belief & $\beta=0.90$ & $30$ [24,37] & $28$ [22,35] & $42$ [35,49]\\
Obstacle & Uniform & Observation & $\beta=0.90$ & $12.5$ [8,18] & $78$ [72,84] & $9.5$ [6,14]\\
Obstacle & Uniform & Carr & -- & $1$ [0,4] & $99$ [96,100] & $0$ [0,2]\\
Obstacle & Uniform & Fwd-Sampling & $\beta=0.85$ & $21.5$ [16,28] & $46$ [39,53] & $32.5$ [26,39]\\
Obstacle & Adversarial & Envelope & $\beta=0.80$ & $19$ [14,25] & $49$ [42,56] & $32$ [26,39]\\
Obstacle & Adversarial & Single-Belief & $\beta=0.90$ & $25$ [19,32] & $29$ [23,36] & $46$ [39,53]\\
Obstacle & Adversarial & Observation & $\beta=0.90$ & $9.5$ [6,14] & $76.5$ [70,82] & $14$ [10,20]\\
Obstacle & Adversarial & Carr & -- & $2$ [1,5] & $98$ [95,99] & $0$ [0,2]\\
Obstacle & Adversarial & Fwd-Sampling & $\beta=0.80$ & $25$ [19,32] & $35.5$ [29,43] & $39.5$ [33,47]\\
\addlinespace[1pt]
CartPole & Uniform & Single-Belief & $\beta=0.95$ & $1.5$ [0,4] & $0$ [0,2] & $98.5$ [96,100]\\
CartPole & Uniform & Observation & $\beta=0.85$ & $1.5$ [0,4] & $0$ [0,2] & $98.5$ [96,100]\\
CartPole & Uniform & Carr & -- & $1.5$ [0,4] & $0$ [0,2] & $98.5$ [96,100]\\
CartPole & Uniform & Fwd-Sampling & $\beta=0.80$ & $0.5$ [0,3] & $0$ [0,2] & $99.5$ [97,100]\\
CartPole & Adversarial & Single-Belief & $\beta=0.95$ & $3$ [1,6] & $0$ [0,2] & $97$ [94,99]\\
CartPole & Adversarial & Observation & $\beta=0.85$ & $1.5$ [0,4] & $0$ [0,2] & $98.5$ [96,100]\\
CartPole & Adversarial & Carr & -- & $1.5$ [0,4] & $0$ [0,2] & $98.5$ [96,100]\\
CartPole & Adversarial & Fwd-Sampling & $\beta=0.60$ & $1.5$ [0,4] & $0$ [0,2] & $98.5$ [96,100]\\
\addlinespace[1pt]
Refuel & Uniform & Single-Belief & $\beta=0.50$ & $0.5$ [0,3] & $37.5$ [31,45] & $62$ [55,69]\\
Refuel & Uniform & Observation & $\beta=0.50$ & $3$ [1,6] & $0$ [0,2] & $97$ [94,99]\\
Refuel & Uniform & Fwd-Sampling & $\beta=0.60$ & $0.5$ [0,3] & $81$ [75,86] & $18.5$ [13,25]\\
Refuel & Adversarial & Single-Belief & $\beta=0.50$ & $2.5$ [1,6] & $34$ [27,41] & $63.5$ [56,70]\\
Refuel & Adversarial & Observation & $\beta=0.50$ & $2$ [1,5] & $0$ [0,2] & $98$ [95,99]\\
Refuel & Adversarial & Fwd-Sampling & $\beta=0.50$ & $2$ [1,5] & $72$ [65,78] & $26$ [20,33]\\
\bottomrule
\end{tabular}
}
\end{table*}

\begin{table*}[t]
\centering
\rev{	\caption{Data underlying Fig.~4 of the main paper. Each entry reports the empirical percentage and a two-sided 95\% Clopper--Pearson interval for the corresponding marginal outcome. For each panel we report the per-method operating point selected by that panel's criterion. Each entry uses $n=200$ rollouts. Interval-belief shields (Single-Belief, Envelope, Fwd-Sampling) are swept over \(\beta\); the Conformal baseline is swept over confidence level \(\kappa\).}
\label{tab:taxinetsim-lowfail-values}
\scriptsize
\begin{tabular}{lllccc}
\toprule
Panel & Shield & Setting & Fail \% & Stuck \% & Safe \%\\
\midrule
Lowest Failure / Uniform & Single-Belief & $\beta=0.95$ & $47$ [37,57] & $15$ [9,24] & $38$ [28,48]\\
 & Envelope & $\beta=0.90$ & $32$ [23,42] & $46$ [36,56] & $22$ [14,31]\\
 & Fwd-Sampling & $\beta=0.80$ & $48$ [38,58] & $26$ [18,36] & $26$ [18,36]\\
 & Conformal & $\kappa=0.995$ & $12.5$ [8,18] & $85.5$ [80,90] & $2$ [1,5]\\
\addlinespace[2pt]
Lowest Failure / Adversarial & Single-Belief & $\beta=0.90$ & $56$ [46,66] & $3$ [1,9] & $41$ [31,51]\\
 & Envelope & $\beta=0.90$ & $39$ [29,49] & $41$ [31,51] & $20$ [13,29]\\
 & Fwd-Sampling & $\beta=0.95$ & $47$ [37,57] & $29$ [20,39] & $24$ [16,34]\\
 & Conformal & $\kappa=0.995$ & $15.5$ [11,21] & $84$ [78,89] & $0.5$ [0,3]\\
\addlinespace[2pt]
Highest Safe / Uniform & Single-Belief & $\beta=0.95$ & $47$ [37,57] & $15$ [9,24] & $38$ [28,48]\\
 & Envelope & $\beta=0.90$ & $32$ [23,42] & $46$ [36,56] & $22$ [14,31]\\
 & Fwd-Sampling & $\beta=0.70$ & $50$ [40,60] & $21$ [13,30] & $29$ [20,39]\\
 & Conformal & $\kappa=0.95$ & $75$ [68,81] & $0$ [0,2] & $25$ [19,32]\\
\addlinespace[2pt]
Highest Safe / Adversarial & Single-Belief & $\beta=0.90$ & $56$ [46,66] & $3$ [1,9] & $41$ [31,51]\\
 & Envelope & $\beta=0.80$ & $40$ [30,50] & $37$ [28,47] & $23$ [15,32]\\
 & Fwd-Sampling & $\beta=0.50$ & $54$ [44,64] & $17$ [10,26] & $29$ [20,39]\\
 & Conformal & $\kappa=0.95$ & $74.5$ [68,80] & $0$ [0,2] & $25.5$ [20,32]\\
\bottomrule
\end{tabular}
}
\end{table*}

\subsection{\texorpdfstring{\rev{Comparison Experiment with Conformal Baseline}}{Comparison Experiment with Conformal Baseline}}
\label{subsec:conformal-perception-model}

\rev{The conformal comparison uses the same TaxiNet-style taxiing dynamics, signed cross-track-error and heading-error state abstraction, action space, and trained controller as the main TaxiNet experiment. It differs only in the perception artifact: the main TaxiNet results use real autonomous-taxiing perception-confusion data from~\cite{pasareanu2023closedloop}, while this comparison uses the simulated taxiing artifact from~\cite{ScarbroImrie25}. That artifact provides paired network point estimates and conformal prediction sets for each true discretized state. The conformal baseline samples such a set and filters actions that would be unsafe in any state in the Cartesian product of the cross-track-error and heading-error sets.}

\rev{The two shield families read different perception inputs: our shields act on a point estimate, while the conformal shield acts on a prediction set. To avoid comparing different perception models, we start from the same trained TaxiNet network evaluation used by the prior conformal work~\cite{ScarbroImrie25} and read both inputs from it. For each true state \(s\), the network's predicted class gives the point-estimate sample \(\hat{s}\), and the conformalizer gives the set sample \(C_\kappa\). From the point-estimate samples we rebuild the TaxiNet \IPOMDP perception model exactly as in the main experiments: for each state-observation pair we estimate \(Z(\hat{s}\mid s)\) and construct the corresponding finite-sample interval bounds. Single-Belief uses the point estimate derived from this model, while Envelope and Fwd-Sampling use the interval bounds during belief propagation.}

\rev{The conformal-set samples define a second, conditional channel \(Q_\kappa(C\mid s,\hat{s})\), implemented axis-wise for cross-track error and heading error. At runtime we first sample the point estimate from the realized point-estimate channel and then sample \(C_\kappa\sim Q_\kappa(\cdot\mid s,\hat{s})\), equivalently \(P(\hat{s},C_\kappa\mid s)=Z(\hat{s}\mid s)Q_\kappa(C_\kappa\mid s,\hat{s})\). This preserves the input correlation and keeps the point estimate inside its conformal set.}

\rev{The uniform and adversarial regimes differ only in the realized point-estimate channel. In the uniform regime, each rollout samples a valid \(Z\) realization inside the rebuilt interval bounds and uses it throughout the trajectory; the conformal set is then drawn conditionally from \(Q_\kappa\). In the adversarial regime, the optimized fixed realization of \(Z\) is used instead, shared with the interval-belief shields. The conformal rates are therefore not adversarially optimized directly, but the conformal input inherits adversarial pressure through the realized point estimate. For Single-Belief, Envelope, and Fwd-Sampling we sweep \(\beta\in\{0.5,0.6,0.7,0.8,0.9,0.95\}\). For the conformal shield we sweep \(\kappa\in\{0.95,0.99,0.995\}\). Each configuration uses 200 Monte Carlo rollouts of horizon 20 under uniform and adversarial perception.}

\begin{figure}[t]
\centering
\includegraphics[width=0.98\columnwidth]{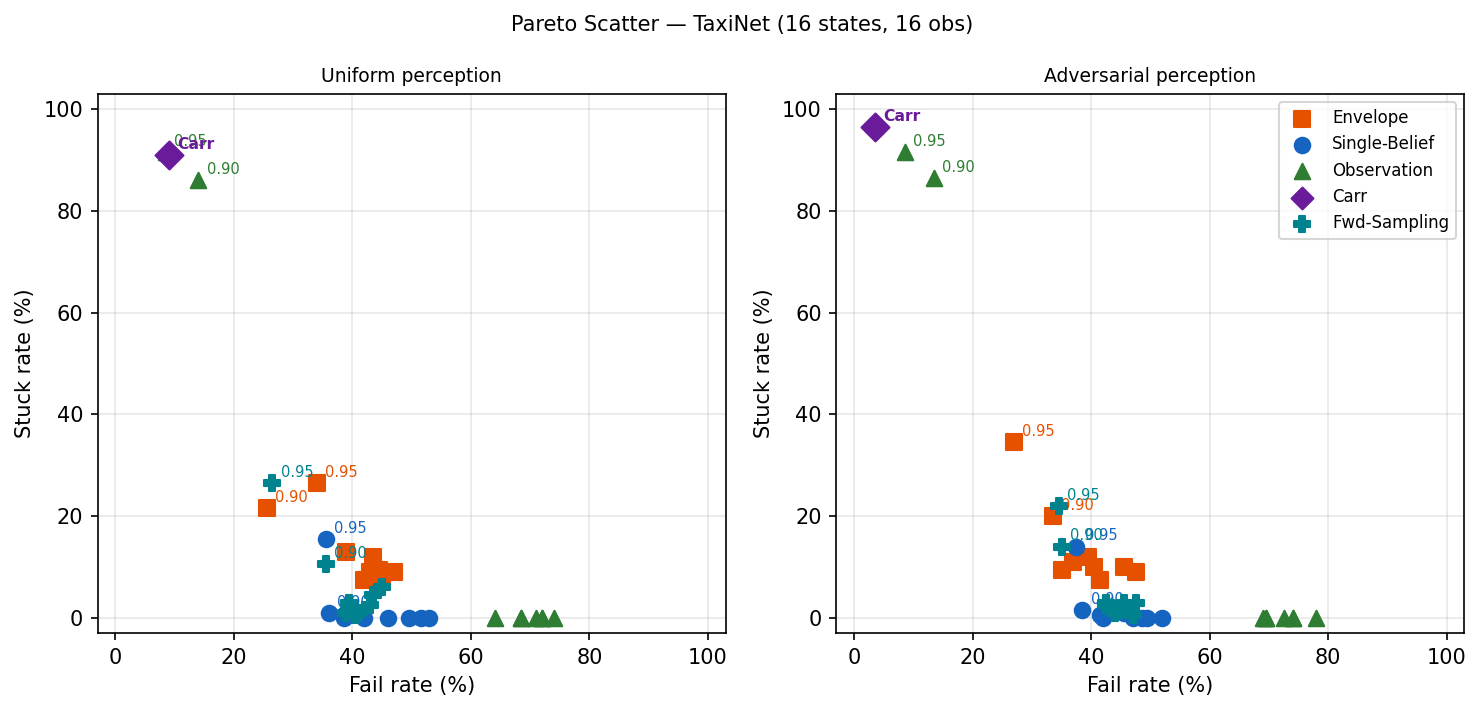}\caption{TaxiNet Pareto scatter over threshold $\beta$ under adversarial perception. History-based shields provide substantially better fail-versus-stuck tradeoffs than memoryless Observation, and the envelope improves on Single-Belief by removing actions that are only safe for the point-estimate model.}
\label{fig:pareto-taxinet-appendix}
\end{figure}

\subsection{Why History Matters: TaxiNet and Obstacle}
\orangechange{TaxiNet demonstrates why the observation count alone is a poor proxy for difficulty. The benchmark has 16 observation labels for 16 states, but the learned perception model still leaves substantial overlap between the posteriors induced by different runway positions. A single noisy observation is therefore consistent with several states that require different steering corrections. Figure~\ref{fig:pareto-taxinet-appendix} shows the resulting threshold sweep: Observation can reduce failures only by blocking almost all actions, whereas all three history-based methods achieve interior tradeoff points. The Envelope Shield dominates Single-Belief in safety because the point-estimate filter is systematically optimistic in the presence of broad observation intervals. Fwd-Sampling still comes much closer to the envelope than Single-Belief does: at its best point it reaches $26.5\%$ fail and $26.5\%$ stuck under uniform perception, versus $25.5\%$ fail and $21.5\%$ stuck for Envelope and $35.5\%$ fail and $15.5\%$ stuck for Single-Belief. By maintaining many belief points and many probability samples, Fwd-Sampling exposes substantially more uncertainty than a single posterior, but as an inner approximation it can still miss the exact worst-case belief corners that the LP envelope captures.}

\orangechange{The remaining TaxiNet failure rate is also informative. Even at $\beta=0.95$, the Envelope Shield still fails in roughly one third of uniform-perception runs and just over one quarter of adversarial-perception runs. This is not obviously an artifact of the LP relaxation alone; rather, the results suggest that it reflects the intrinsic difficulty of a runway-alignment task in which observation noise still leaves several plausible latent states after one sensor reading. The shield can prevent overconfident corrections, but it cannot recover information that the sensor channel never supplied.}

\begin{figure}[t]
\centering
\includegraphics[width=0.98\columnwidth]{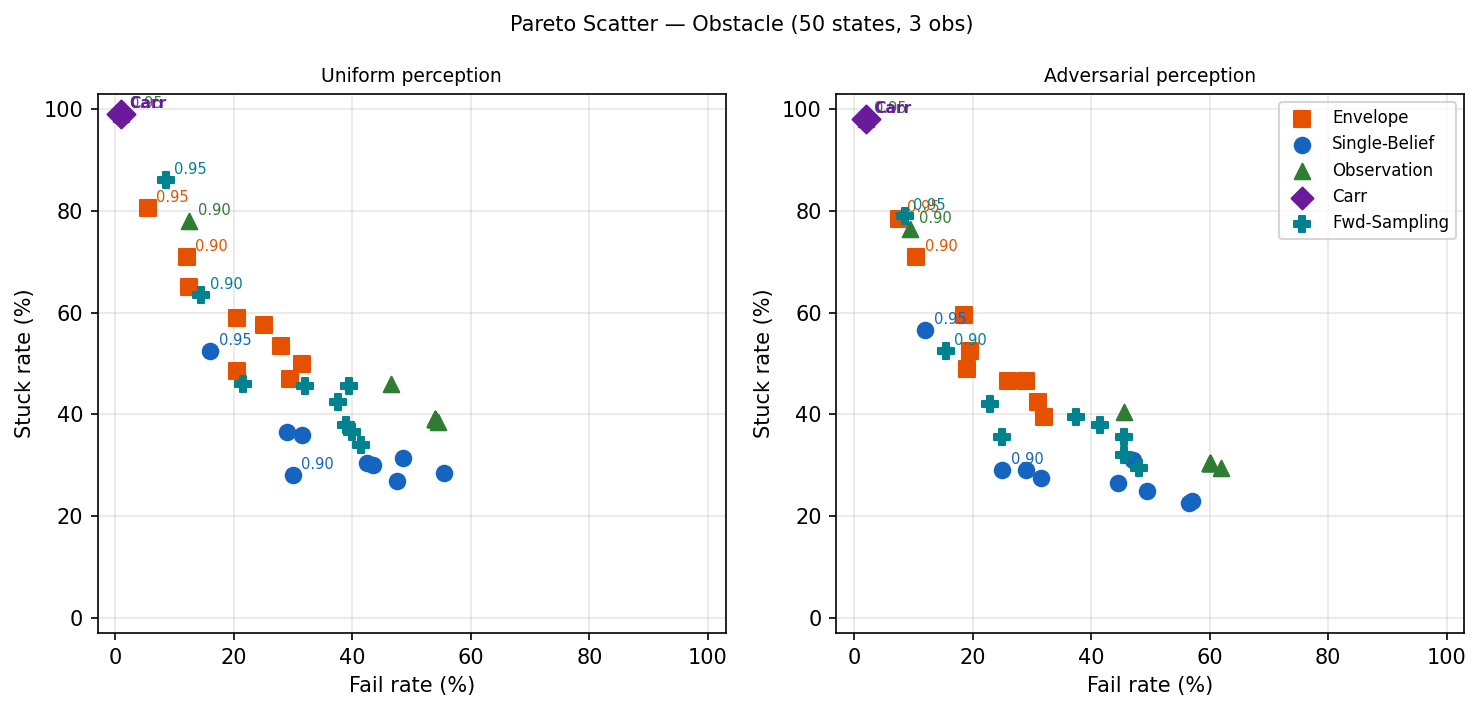}\caption{Obstacle Pareto scatter over threshold $\beta$ under adversarial perception. With only three observations for fifty states, memoryless posterior reasoning collapses toward the same extreme point as support-only shielding. Belief-history remains essential, and interval reasoning improves safety at every threshold.}
\label{fig:pareto-obstacle}
\end{figure}

Obstacle pushes the aliasing effect further. With only 3 observations for 50 states, each observation aggregates many states with incompatible safe actions. Figure~\ref{fig:pareto-obstacle} shows that memoryless Observation and Carr are then driven to the same corner of the tradeoff curve: around $2\%$ fail, but at roughly $98\%$ stuck. Belief history breaks this deadlock because the action-observation sequence rules out large parts of the support even when the current observation does not. Single-Belief exploits that history, Fwd-Sampling adds a substantial extra layer of conservatism, and the envelope improves on both at every threshold by treating interval uncertainty adversarially rather than as a point estimate. At the best threshold, Fwd-Sampling reaches $7\%$ fail and $84\%$ stuck, much closer to Envelope's $3\%/85\%$ than to Single-Belief's $14\%/50\%$. This case suggests that in heavily aliased perception problems, quantitative belief-mass reasoning is not just more fine grained than support reasoning; it changes which shields are usable at all.

\subsection{Near-Bijective and Large-Scale Regimes}
CartPole and Refuel delimit the other two operating regimes. CartPole uses deliberately degraded perception, but it still has 82 observation labels for 82 states and only two actions. As a result, there are few opportunities for the shield to become stuck: at the best threshold, Single-Belief achieves about $2\%$ fail and $0\%$ stuck, Observation about $3\%$ fail and $0\%$ stuck, Carr about $3\%$ fail and $0\%$ stuck, and Fwd-Sampling about $2\%$ fail and $0\%$ stuck in the uniform regime. In this regime, the first observation already localizes the state well enough that history adds little. The experiment therefore acts as a sanity check: when interval uncertainty does not create substantial ambiguity, the heavier envelope machinery is unnecessary.

Refuel is the opposite kind of stress test. The observation model intentionally hides the safety-critical variables, so danger must be inferred indirectly from history rather than read from a single observation. The Envelope Shield is excluded there because LP-based online propagation is infeasible at this scale, and Carr is excluded because our modified Refuel observation model hides the safety-critical variables, which greatly increases support aliasing and makes the offline support-MDP breadth-first exploration intractable. Still, the comparison among the three feasible baselines is revealing. Single-Belief is the strongest lightweight zero-failure shield, reaching $0\%$ fail at about $79\%$ stuck under uniform perception and $58\%$ stuck under the adversarial regime. Observation is more permissive and uniquely attains a $0\%$-stuck operating point at lower thresholds, but does so by accepting a nontrivial failure rate ($3$--$4.5\%$). Fwd-Sampling does not match Single-Belief on Refuel: at its best low-failure point it reaches $0\%$ fail but about $98\%$ stuck under uniform perception and $93\%$ stuck under adversarial perception. The lesson is not that memoryless or sampled shielding is preferable on Refuel; rather, it shows that large hidden-state problems force a three-way tradeoff between safety, stuck avoidance, and online computational budget.

\subsection{\texorpdfstring{\rev{Sensitivity to the Confidence Level}}{Sensitivity to the Confidence Level}}
\rev{The significance level $\alpha$ controls the width of the Clopper--Pearson intervals used to build the \IPOMDP: smaller $\alpha$ yields wider intervals and therefore exposes the shield to a larger set of admissible observation kernels. To isolate this effect, we ran an additional TaxiNet sweep over $\alpha\in\{0.01,0.025,0.05,0.075,0.10,0.15,0.20,0.30\}$ and $\beta\in\{0.7,0.8,0.9\}$ using Single-Belief, Envelope, and Fwd-Sampling. Each configuration aggregates 200 trials, and the adversarial observation realization is optimized against the Envelope shield at the same $\beta$ and then reused for the other shields.}

\begin{figure*}[t]
\centering
\includegraphics[width=0.92\textwidth]{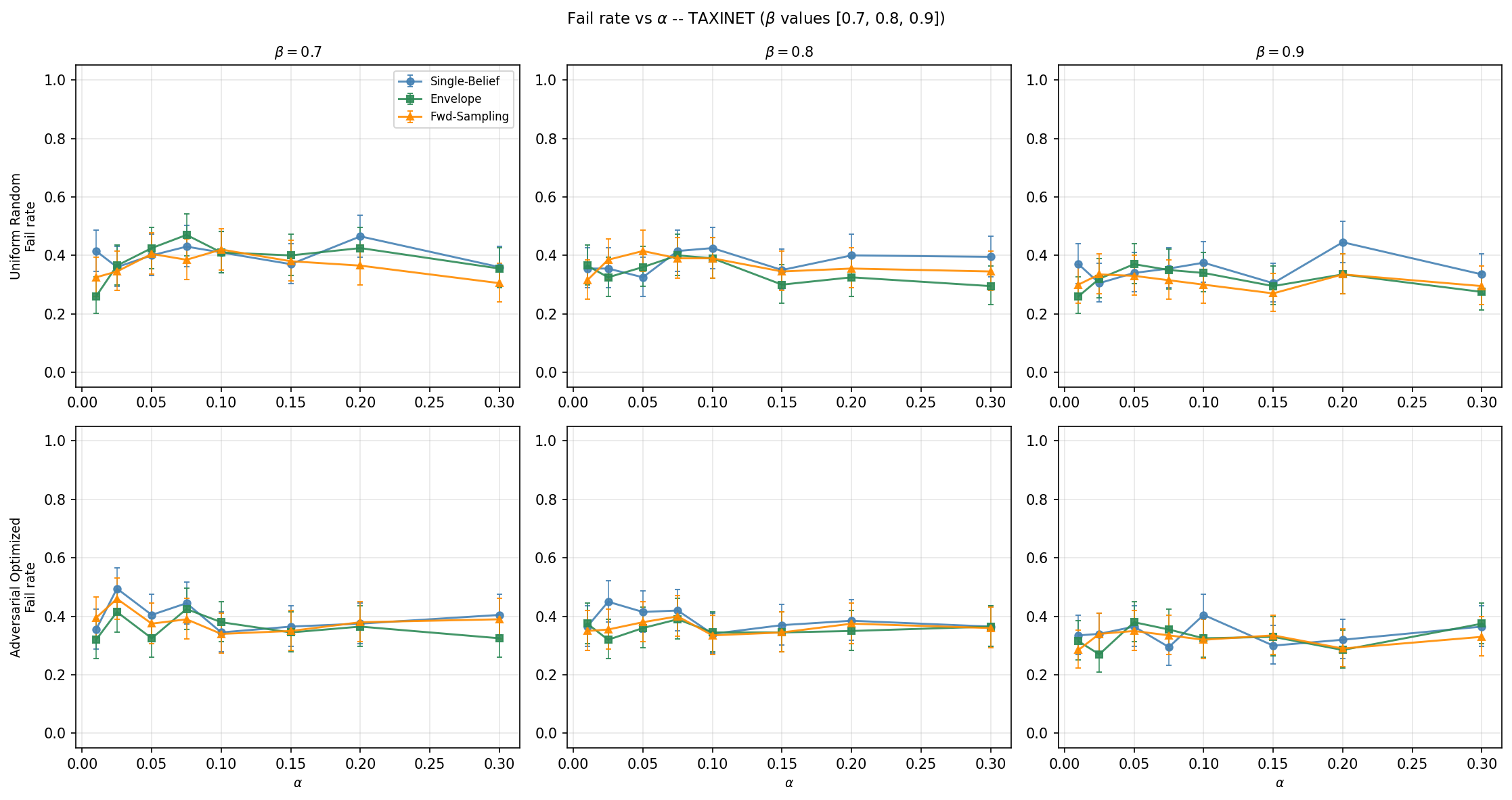}
	\caption{\rev{TaxiNet fail rate as a function of $\alpha$ over three shield thresholds. Rows show the uniform and adversarial perception regimes; columns show $\beta\in\{0.7,0.8,0.9\}$. Error bars show 95\% Clopper--Pearson intervals over 200 trials per configuration.}}
\label{fig:alpha-sweep-taxinet-fail}
\end{figure*}

\begin{figure*}[t]
\centering
\includegraphics[width=0.92\textwidth]{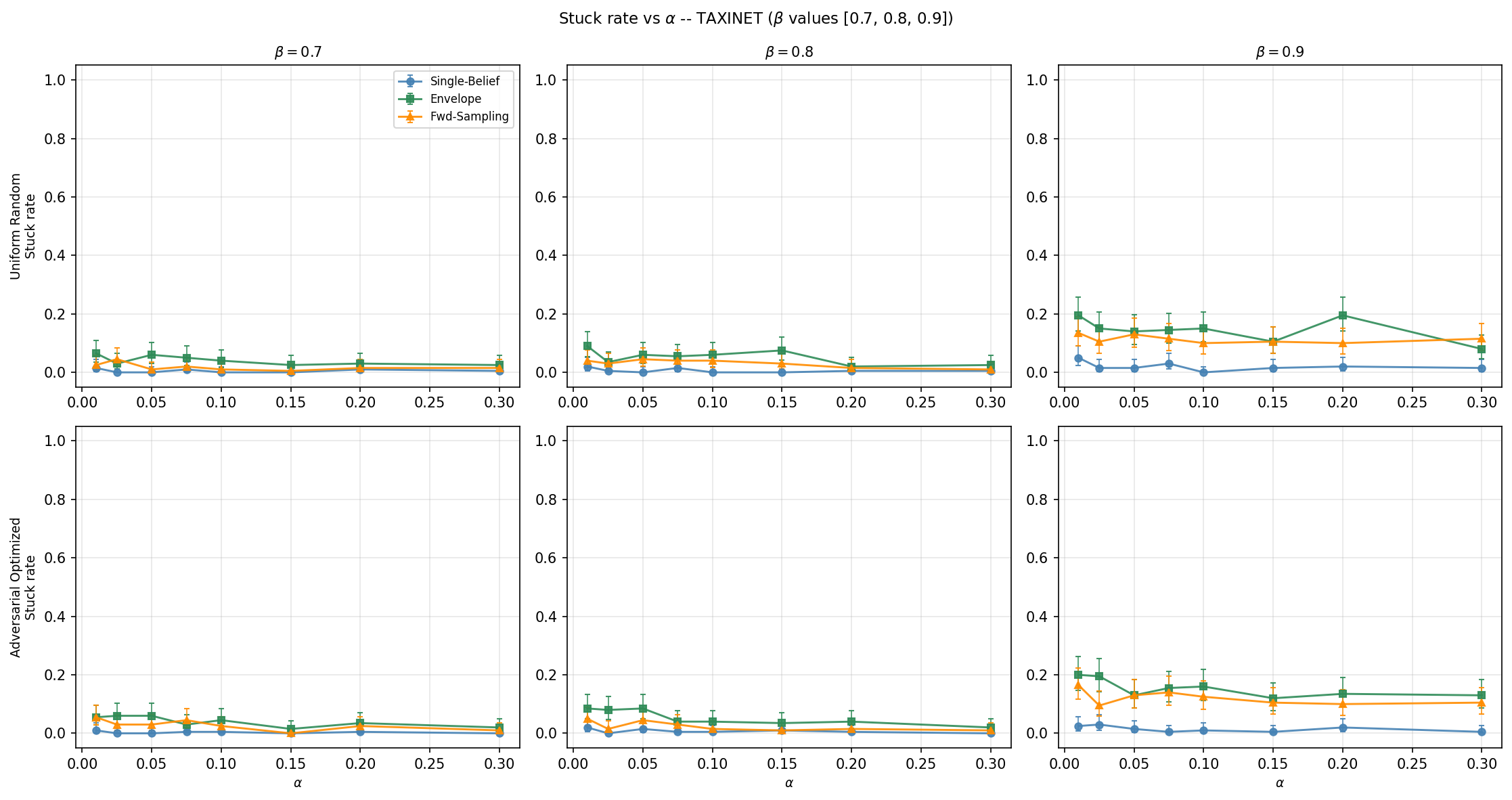}
\caption{\rev{TaxiNet stuck rate as a function of $\alpha$ over three shield thresholds. Rows show the uniform and adversarial perception regimes; columns show $\beta\in\{0.7,0.8,0.9\}$. Error bars show 95\% Clopper--Pearson intervals over 200 trials per configuration.}}
\label{fig:alpha-sweep-taxinet-stuck}
\end{figure*}

\begin{figure*}[t]
\centering
\includegraphics[width=0.92\textwidth]{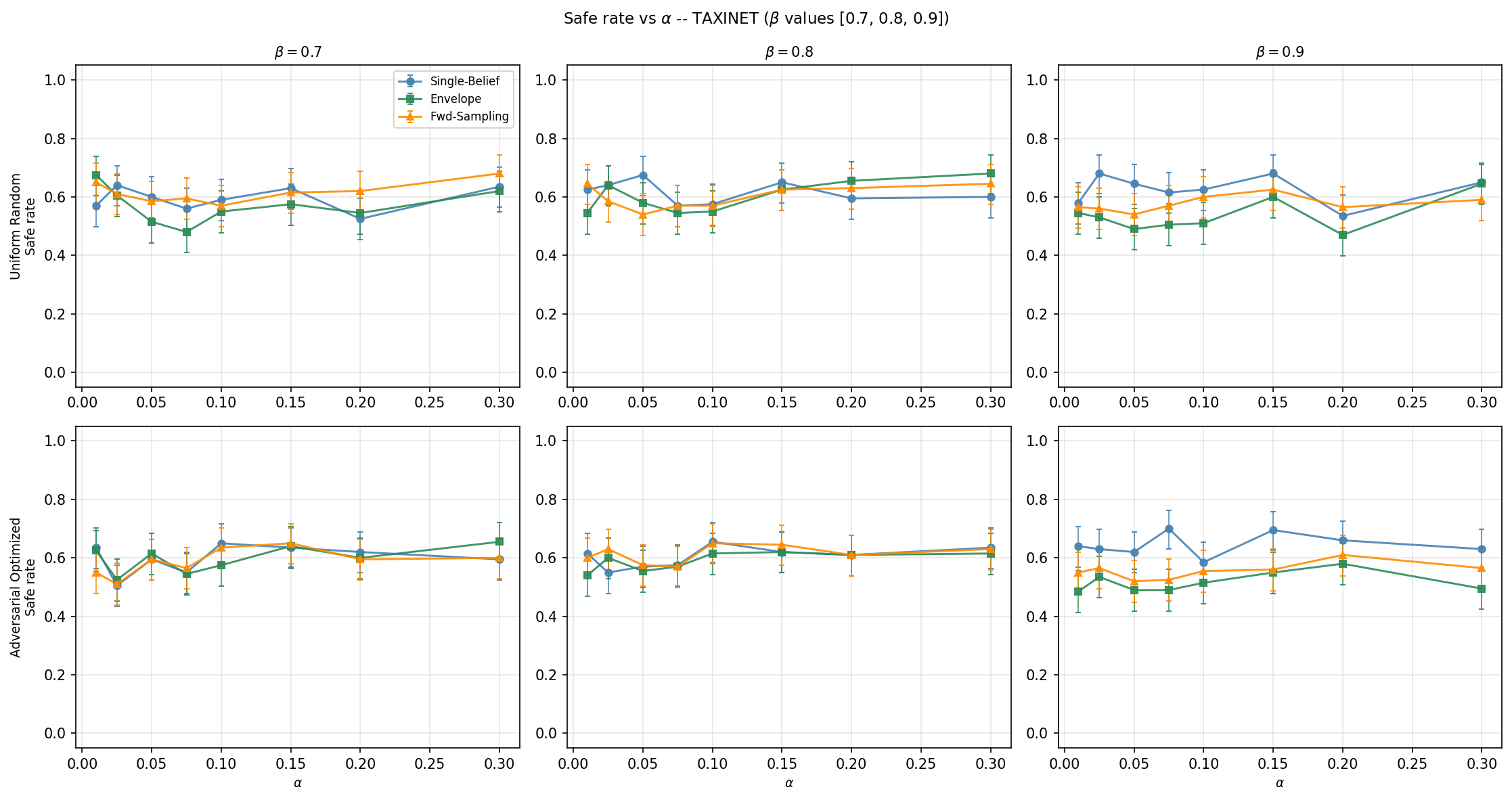}
\caption{\rev{TaxiNet safe-completion rate as a function of $\alpha$ over three shield thresholds. Rows show the uniform and adversarial perception regimes; columns show $\beta\in\{0.7,0.8,0.9\}$. Error bars show 95\% Clopper--Pearson intervals over 200 trials per configuration.}}
\label{fig:alpha-sweep-taxinet-safe}
\end{figure*}

\rev{Figures~\ref{fig:alpha-sweep-taxinet-fail}--\ref{fig:alpha-sweep-taxinet-safe} show that $\beta$ is the dominant tuning knob. At $\beta=0.7$ and $\beta=0.8$, the method ordering and fail--stuck tradeoff change only modestly as $\alpha$ varies, and the stuck rates for Single-Belief remain near zero. At $\beta=0.9$, $\alpha$ has a visible but still secondary effect: under adversarial perception the lowest-failure Envelope configuration occurs at \(\alpha=0.025\), reaching \(27.0\%\) fail, \(19.5\%\) stuck, and \(53.5\%\) safe; Fwd-Sampling's lowest-failure point occurs at \(\alpha=0.01\), reaching \(28.5\%\) fail, \(16.5\%\) stuck, and \(55.0\%\) safe; and Single-Belief reaches \(29.5\%\) fail, \(0.5\%\) stuck, and \(70.0\%\) safe at \(\alpha=0.075\). Thus the refreshed sweep narrows the gap among the three shields in this auxiliary TaxiNet setting and makes the safety--liveness tradeoff more visible: interval methods still reduce failures slightly at high \(\beta\), but Single-Belief retains substantially lower stuck rate. Overall, the sweep supports the interpretation that $\alpha$ mainly fine-tunes interval conservatism within a region selected by $\beta$, while the broad operating regime is governed by the runtime shield threshold.}

\subsection{Measured Conservatism of the Envelope}
The runtime envelope is intentionally an over-approximation, so an important question is whether its conservatism is measurable rather than merely asserted. \rev{A second, related question is how much of that conservatism comes from over-approximating the reachable belief set, and how much comes from the interval semantics for perception. In that semantics, the observation kernel \(Z_t\) used in each Bayes update may vary with time, provided every row remains row-stochastic and within the learned bounds \([Z^-,Z^+]\). This is the semantics used by the LFP envelope and by Fwd-Sampling; it is more general than assuming that one admissible perception kernel is sampled once and reused for the whole trajectory.} We study \rev{both} on TaxiNet by comparing the LFP envelope against \rev{two different} forward-sampled under-approximations of reachable beliefs\rev{: a \emph{time-varying} sampler that resamples admissible perception probabilities at each belief update, and a \emph{fixed-probability} sampler that draws one admissible perception kernel and reuses it for all belief updates along that trajectory. The Monte Carlo rollout histories are the same in both diagnostics; only the forward-sampled belief set used for comparison changes.} For each action under evaluation $a$ and belief $b$, let $P(b,\mathrm{allowed}(a))$ denote the belief-weighted mass on true states for which that action is permitted by the perfect-perception shield. Label $\mathcal{B}^{Fwd}$ the belief set constructed through forward sampling \rev{(in either variant)} and recall $\hat{\mathcal{B}}$ is constructed through LFP. For each trajectory \rev{and each sampler} we compute
\[
\mathrm{gap}= \min_{b\in\mathcal{B}^{Fwd}}~P(b,\mathrm{allowed(a)})-\min_{b\in\hat{\mathcal{B}}}~P(b,\mathrm{allowed(a)})\ge 0,
\]
where the sampled set under-approximates the true reachable beliefs and the LFP polytope over-approximates them. Consequently, this measured gap should be read as a sampling-based diagnostic rather than an exact coarseness measure: because the sampled set may miss additional low-score beliefs, it need not coincide with the gap to the exact reachable set.

\begin{figure}[t]
\centering
\includegraphics[width=0.98\columnwidth]{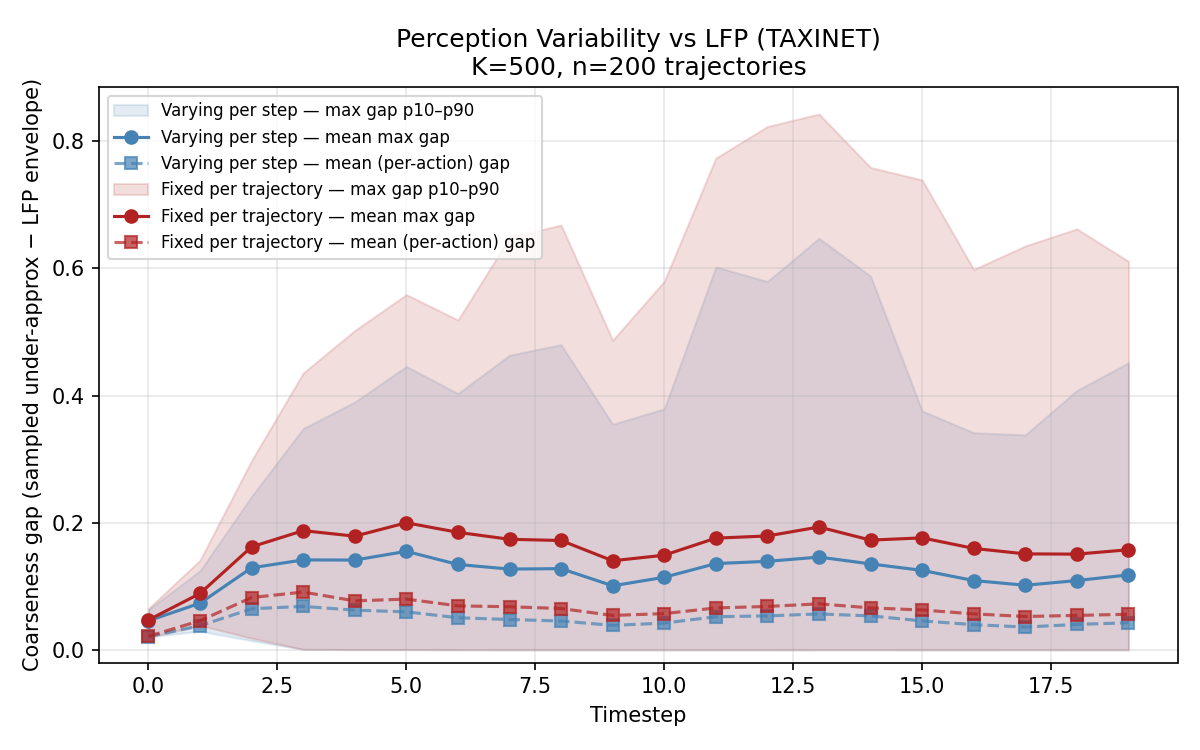}
\caption{\rev{LFP envelope conservativeness on TaxiNet, measured against two forward-sampled under-approximations of reachable beliefs. The \emph{time-varying} sampler resamples admissible perception probabilities at every step; the \emph{fixed-probability} sampler draws one admissible perception model and holds it fixed for the whole trajectory. Each curve plots the gap between its under-approximation and the LFP envelope, so each estimates how loose the envelope is relative to that sampler, and the spread between the curves estimates the additional envelope looseness attributable to protecting against time-varying perception.}}
\label{fig:perception-variability}
\end{figure}

\rev{Figure~\ref{fig:perception-variability} reports the envelope's conservativeness on 200 TaxiNet trajectories, measured against each of the two forward-sampled under-approximations. At each time step, the plotted value is the mean gap over the shield action evaluations performed at that step. These per-step mean gaps stay roughly in the $0.05$--$0.20$ range for the fixed-probability sampler and $0.05$--$0.15$ for the time-varying sampler, so the envelope's looseness is bounded but non-negligible. Aggregated over trajectories, the gap to the fixed-probability under-approximation has mean $0.541$, median $0.464$, and $90$th percentile $0.994$, while the gap to the time-varying under-approximation has mean $0.443$, median $0.357$, and $90$th percentile $0.980$; the corresponding mean per-action gaps are $0.0626$ and $0.0475$.}

\rev{Reading the two curves together separates two sources of envelope conservativeness. The time-varying under-approximation uses the same perception semantics as the LFP envelope, so its gap to the envelope estimates the looseness intrinsic to the LFP over-approximation of the reachable belief set. The fixed-probability under-approximation uses a stricter modeling assumption: one admissible perception kernel is reused throughout the trajectory, as may be more appropriate when the learned intervals represent uncertainty about stationary perception probabilities rather than step-to-step variability. Since this fixed-kernel reachable set is smaller, its gap to the same envelope is larger; the excess over the time-varying gap estimates the additional pessimism from using a shield construction that protects against time-varying \(Z_t\) on a problem better modeled by a single admissible \(Z\).}

\section{Inference-Time Analysis}
\label{sec:inference-timing}

\begin{figure}[t]
\centering
\includegraphics[width=0.98\columnwidth]{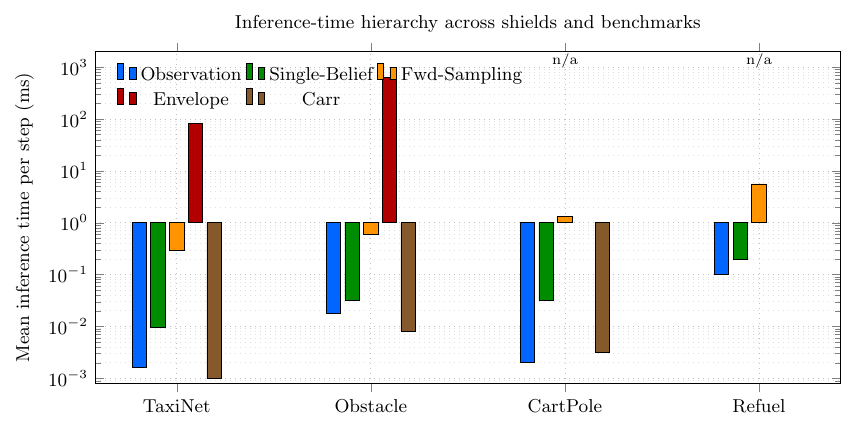}
\caption{Mean shield inference time per step on a log scale. Observation, Single-Belief, and Carr remain in the microsecond regime; Fwd-Sampling moves into the low-millisecond regime; and the LP-based Envelope is substantially slower, with CartPole and Refuel beyond the feasible range.}
\label{fig:timing}
\end{figure}

\orangechange{The fail-versus-stuck tradeoffs in Figs.~\ref{fig:summary-bars} and~\ref{fig:summary-bars-safe} are mirrored by the runtime hierarchy in Fig.~\ref{fig:timing}. These timings measure one shield inference call, including belief propagation and action filtering but excluding controller inference, environment stepping, and observation sampling, on an AMD EPYC 9454P. Observation, Single-Belief, and Carr are in the microsecond regime; Fwd-Sampling is in the millisecond-to-subsecond regime; and LP-based Envelope is two to five orders slower where feasible. On TaxiNet, for example, mean latency is \(9.4\,\mu\)s for Single-Belief, \(21.2\) ms for Fwd-Sampling, and \(83.1\) ms for Envelope. Refuel is the scalability boundary: Observation and Single-Belief remain sub-millisecond, Fwd-Sampling reaches \(383.9\) ms, and Envelope and Carr are infeasible.}

\orangechange{These timings sharpen the engineering interpretation. Single-Belief appears to be the strongest lightweight default in these benchmarks: nearly as cheap as Observation while retaining history and generally outperforming memoryless shielding. Fwd-Sampling is a useful middle ground on smaller models, much faster than Envelope on TaxiNet and Obstacle while recovering much of its safety benefit, but its cost grows steeply with state dimension and interval coverage. Envelope is best viewed as a low-frequency or smaller-model mechanism for hard aliased regimes, while Fwd-Sampling is the scalable approximation when exact envelope propagation is unavailable but more conservatism than Single-Belief is desired.}

\section{Proofs}
\label{app:proofs}

\begin{proof}[Proof of Theorem~V.1]
We first prove by induction on $t$ that $\mathcal{B}_t\subseteq\widehat{\mathcal{B}}_t$ for all $t$.

Base case: by construction, $\mathcal{B}_0\subseteq\widehat{\mathcal{B}}_0$.

Inductive step: assume $\mathcal{B}_t\subseteq\widehat{\mathcal{B}}_t$. Consider any exact posterior belief $b'\in\mathcal{B}_{t+1}$. By definition of the concrete reachable-belief set, $b'$ is obtained by applying one exact action-observation update from some prior belief in $\mathcal{B}_t$ using a perception uncertainty model in $\mathcal{Z}$. Since the induction hypothesis gives $\mathcal{B}_t\subseteq\widehat{\mathcal{B}}_t$, and since the propagated envelope is constructed to contain every such exact one-step posterior, we obtain $b'\in\widehat{\mathcal{B}}_{t+1}$. Hence $\mathcal{B}_{t+1}\subseteq\widehat{\mathcal{B}}_{t+1}$.

Therefore $\mathcal{B}_t\subseteq\widehat{\mathcal{B}}_t$ holds for every $t$.

Now fix any $Z\in\mathcal{Z}$. The posterior belief induced by history $h_t$ is some $b_t^Z\in\mathcal{B}_t$. Under that belief,
\[
\Pr_Z[a\in\Omega(s_t)\mid h_t]=\sum_{s\in S} b_t^Z(s)\mathbf{1}[a\in\Omega(s)]=\phi_a(b_t^Z).
\]
Therefore,
\[
\inf_{Z\in\mathcal{Z}}\Pr_Z[a\in\Omega(s_t)\mid h_t]
=
\inf_{b\in\mathcal{B}_t}\phi_a(b).
\]
Since $\mathcal{B}_t\subseteq\widehat{\mathcal{B}}_t$, minimizing over the larger set $\widehat{\mathcal{B}}_t$ can only decrease the value, so
\[
\inf_{b\in\mathcal{B}_t}\phi_a(b)\ge \inf_{b\in\widehat{\mathcal{B}}_t}\phi_a(b)=\underline{p}_t(a).
\]
If $a\in\widehat{\Xi}(h_t)$, then by definition $\underline{p}_t(a)\ge\beta$, which yields the claim.
\end{proof}

\begin{proof}[Proof of Theorem~V.2]
Condition on the event $\mathcal{E}_{\mathrm{corr}}$, so that the true perception uncertainty model satisfies $Z^\ast\in\mathcal{Z}$. For $t=0,\dots,H$, let
\[
E_t=\{s_0,\dots,s_t\in C\},
\]
so $\mathrm{Safe}_{0:H}=E_H$. By assumption, the initial support lies in $C$, hence $\Pr(E_0\mid \mathcal{E}_{\mathrm{corr}})=1$.

Fix any $t\in\{0,\dots,H-1\}$. On the event $E_t$, we have $s_t\in C$. Since actions are chosen from $\widehat{\Xi}$, the executed action $a_t$ satisfies $a_t\in\widehat{\Xi}(h_t)$. The Abstraction Soundness theorem then implies
\[
\Pr[a_t\in\Omega(s_t)\mid h_t,\mathcal{E}_{\mathrm{corr}}]\ge \beta.
\]
Whenever $s_t\in C$ and $a_t\in\Omega(s_t)$, Assumption~V.1 gives
\[
\Pr[s_{t+1}\in C\mid s_t,a_t]\ge \gamma.
\]
Combining these two bounds,
\[
\Pr[s_{t+1}\in C\mid h_t,E_t,\mathcal{E}_{\mathrm{corr}}]
\ge
\Pr[a_t\in\Omega(s_t)\mid h_t,\mathcal{E}_{\mathrm{corr}}]\cdot \gamma
\ge \beta\gamma.
\]
Therefore,
\[
\Pr(E_{t+1}\mid E_t,\mathcal{E}_{\mathrm{corr}})\ge \beta\gamma.
\]
Applying the chain rule over $t=0,\dots,H-1$ yields
\[
\Pr(E_H\mid \mathcal{E}_{\mathrm{corr}})
=
\prod_{t=0}^{H-1}\Pr(E_{t+1}\mid E_t,\mathcal{E}_{\mathrm{corr}})
\ge
(\beta\gamma)^H.
\]
Since $\mathrm{Safe}_{0:H}=E_H$, this is exactly the desired inner bound.

Finally, by construction of the confidence set, $\Pr(\mathcal{E}_{\mathrm{corr}})\ge \lambda$. Equivalently, with outer probability at least $\lambda$ over the random training dataset, the conditional closed-loop safety probability is at least $(\beta\gamma)^H$. This is exactly
\[
\Pr\!\left[\Pr\left[\mathrm{Safe}_{0:H}\right]\ge (\beta\gamma)^H\right]\ge\lambda.
\]
\end{proof}

\section{Categorical Foundations: Credal Coalgebraic Typing and Belief Semantics}
This section gives a \emph{categorical typing story} for interval perception uncertainty and explains how it directly motivates our choice of a convex (LP-friendly) abstract domain. The key observation is that interval constraints induce \emph{convex} sets of admissible distributions (credal sets), and belief-state semantics is governed by convex-algebra structure \cite{bonchi2021distribution}. Our abstract interpretation instantiates this structure computationally: it approximates credal belief sets by finitely presented convex sets and evaluates admissibility by extremizing linear functionals, which are the natural observables in a convex-algebraic setting.

\subsection{Distributions, Convex Powersets, and Credal Sets}
Let $\mathbf{Conv}$ denote the category of convex algebras and affine maps, and let $U:\mathbf{Conv}\to\mathbf{Sets}$ be the forgetful functor. For a finite set $X$, $\mathcal{D}(X)$ denotes the simplex of probability distributions on $X$, viewed as the free convex algebra on $X$ \cite{bonchi2021distribution}. Let $\mathcal{P}_c$ denote the convex-powerset construction on convex algebras. On underlying sets, $U\mathcal{P}_c(\mathcal{D}(X))$ is represented by the nonempty convex subsets of $\mathcal{D}(X)$; we write this carrier as $\mathcal{K}(X)$. Thus, in the discussion below, a ``credal set'' is precisely an element of $\mathcal{K}(X)$.

\subsection{A Coalgebraic Type for an \IPOMDP}
For a POMDP it is convenient to bundle transition and emission into a one-step \emph{joint} distribution over $(s',o)\in S\times O$. Under interval uncertainty, the perception uncertainty model is no longer a single kernel but a set $\mathcal{Z}$; in the present section it is most natural to view this as a \emph{nondeterministic perception uncertainty model}. Consequently, from a given $(s,a)$ we obtain a \emph{credal set} of admissible joint distributions. The point of this typing is not to replace the LP construction, but to isolate the structural reason the LP construction works: before normalization, the one-step semantics is a convex set of distributions, and shield queries depend only on linear functionals of beliefs. That is exactly the regime in which convex envelope abstractions and LP extremization are natural.

Concretely, for each $Z\in\mathcal{Z}$ define the joint kernel
\[
P_Z(s',o\mid s,a)\triangleq T(s,a,s')\,Z(o\mid s').
\]
The family $\{P_Z(\cdot\mid s,a):Z\in\mathcal{Z}\}\subseteq \mathcal{D}(S\times O)$ is convex because $\mathcal{Z}$ is a convex polytope (interval constraints intersected with simplex constraints) and $Z\mapsto P_Z(\cdot\mid s,a)$ is affine. Following the pattern of \cite{bonchi2021distribution}, it is useful to separate the abstract typed definition from its set-level concretization.

At the typed level, an \IPOMDP\ is presented by a coalgebra
\[
c:\;S\to \big(\mathcal{P}_c(\mathcal{D}(S\times O))\big)^A,
\]
whose codomain says that each state and action yields a convex choice of one-step joint distributions over next-state/observation pairs.
On underlying sets, the codomain of $c$ is represented by $\mathcal{K}(S\times O)$. But the concrete propagated object of interest is not this state-level map itself; it is the induced belief-state transformer obtained by lifting $c$ along the convex-algebra structure on $\mathcal{D}(S)$.

\subsection{Belief-Space Semantics and Convex Closure}
Beliefs live in $\mathcal{D}(S)$, which carries the free convex-algebra structure used in distribution-first presentations of probabilistic systems \cite{bonchi2021distribution}. In the same spirit as the belief-state transformer $c^\sharp:\mathcal{D}(S)\to (\mathcal{P}\mathcal{D}(S))^L$ for probabilistic automata, the state-level \IPOMDP coalgebra induces an exact concrete transformer on single beliefs,
\[
\hat c:\mathcal{D}(S)\to \big(\mathcal{K}(S)\big)^{A\times O},
\]
defined by
\[
\begin{aligned}
\hat c(b)(a,o)=\biggl\{\;
&\operatorname{Bayes}_o\!\left(\sum_{s\in S} b(s)\,P_Z(\cdot\mid s,a)\right)\\
&\biggm|\; Z\in\mathcal{Z},\ \Pr_Z(o\mid b,a)>0 \;\biggr\}.
\end{aligned}
\]
Thus the left-hand side changes from states to beliefs, exactly because propagation takes convex combinations of the state-indexed one-step semantics. This is the direct analogue of Bonchi's determinized belief-state transformer: a single current belief is mapped to a set of possible successor beliefs for each action-observation pair.

Equivalently, one may factor this through the intermediate joint-distribution credal set
\[
\mathsf{Joint}(b,a)=\left\{
\sum_{s\in S} b(s)\,P_Z(\cdot\mid s,a): Z\in\mathcal{Z}
\right\}\subseteq \mathcal{D}(S\times O),
\]
and then write $\hat c(b)(a,o)=\mathsf{Post}(b,a,o)$ for the posterior set obtained by conditioning $\mathsf{Joint}(b,a)$ on $o$, where $\Pr_Z(o\mid b,a)$ denotes the corresponding marginal observation probability. Affine prediction preserves convexity; the essential difficulty is that the normalization in $\operatorname{Bayes}_o$ can destroy convexity of $\hat c(b)(a,o)$.

For shielding, however, we propagate not one belief but a reachable set of beliefs. This gives the concrete set transformer
\[
\begin{gathered}
\mathcal{F}:\mathcal{K}(S)\to \big(\mathcal{K}(S)\big)^{A\times O},\\
\mathcal{F}(B)(a,o)=\operatorname{conv}\!\left(\bigcup_{b\in B}\hat c(b)(a,o)\right),
\end{gathered}
\]
where by $\mathcal{K}(S)$ we mean nonempty convex subsets of $\mathcal{D}(S)$. Conditioning on an observation $o$ can destroy convexity even when $B$ is convex, so the explicit convex closure in $\mathcal{F}$ is exactly the ``up to convex hull'' step used by our envelope semantics. Our method therefore propagates the envelope $\widehat{\mathcal{B}}_t$ rather than the exact generally nonconvex posterior set, so that linear functionals of beliefs (such as $\phi_a$) can be optimized soundly by LP.

\subsection{Normalization as a Cone Projection and the Role of Charnes--Cooper}
The Bayesian posterior update includes division by a normalizing constant. A useful perspective is to perform the update first in the positive cone of \emph{unnormalized} measures (or subdistributions), where the mapping is affine, and then project back to the simplex by normalization. This is exactly the computational role of our LFP lifting and Charnes--Cooper step: it is a controlled way to express normalization as linear-fractional constraints and reduce extremal posterior queries to LP without changing their optimal values.

\subsection{Abstract Interpretation Guided by Convex Algebra}
Abstract interpretation studies sound finite representations of typically infinite-state semantics by relating a concrete domain to an abstract domain through abstraction and concretization maps. In our setting, the concrete domain consists of reachable convex belief sets, while the abstract domain consists of finitely represented envelopes such as template polytopes. The abstraction map $\mathsf{abs}$ takes a concrete reachable belief set and returns a sound over-approximation in the chosen template domain; the concretization map $\mathsf{conc}$ interprets that abstract element back as the set of beliefs satisfying the template constraints. The abstract transformer $\mathcal{F}^\sharp$ is exactly our LFP/LP propagation pipeline.

From this perspective, the soundness condition
\[
\mathcal{F}(B)\subseteq \mathsf{conc}(\mathcal{F}^\sharp(\mathsf{abs}(B)))
\]
says that if $B$ is a concrete set of beliefs reachable before an action-observation update, then the beliefs actually reachable after the update are contained in the concretization of the abstract envelope produced by our propagation step. Template polytopes are one concrete realization of this recipe: they preserve affine prediction, admit conservative relaxations for the bilinear coupling introduced by uncertain observation likelihoods, and reduce extremal query evaluation to LP.

This viewpoint also clarifies portability. The method does not rely on the interval perception uncertainty model being rectangular. More generally, it applies to any perception uncertainty model represented by a convex polytope of admissible observation probabilities: the prediction step remains affine, the uncertain observation update can still be lifted into a convex belief-envelope computation, and runtime admissibility is still obtained by optimizing linear objectives over that envelope.

\end{document}